%% file: main.tex
\documentclass[twoside]{article}

\usepackage[preprint]{aistats2026}
\usepackage[round]{natbib}

\usepackage[utf8]{inputenc} 
\usepackage[T1]{fontenc}    
\usepackage{hyperref}       
\usepackage{url}            
\usepackage{booktabs}       
\usepackage{amsfonts}       
\usepackage{nicefrac}       
\usepackage{microtype}      
\usepackage[table]{xcolor}  
\usepackage{array}
\usepackage{amsmath}
\usepackage{amsthm}
\usepackage{amssymb}
\usepackage{graphicx}
\usepackage{spverbatim}
\usepackage{graphicx}
\usepackage{subcaption}
\usepackage{todonotes}
\setuptodonotes{inline}
\usepackage{multirow}

\theoremstyle{plain}
\newtheorem{theorem}{Theorem}[section]
\newtheorem{proposition}[theorem]{Proposition}
\newtheorem{lemma}[theorem]{Lemma}
\theoremstyle{definition}
\newtheorem{definition}[theorem]{Definition}
\newtheorem{remark}[theorem]{Remark}
\newcommand{\TAU}{{\color[rgb]{0.0,0.0,0.0}\LARGE\tau}}
\newcommand{\R}{{\color[rgb]{0.0,0.0,0.0}R}}
\newcommand{\Rmin}{{\color[rgb]{0.0,0.0,0.0}R^*_{\color[rgb]{0.0,0.0,0.0}\kappa}}}
\newcommand{\f}{{\color[rgb]{0.0,0.0,0.0}\kappa}}
\newcommand{\V}{{\color[rgb]{0.0,0.0,0.0}V}}

\begin{document}

%

%

\twocolumn[

\aistatstitle{
A Consequentialist Critique of Binary Classification Evaluation: Theory, Practice, and Tools}

\aistatsauthor{
  Gerardo Flores
  \And
  Abigail Schiff
  \And
  Alyssa H. Smith
  \And
  Julia A. Fukuyama
  \And
  Ashia C. Wilson
}

\aistatsaddress{
  MIT
  \And
  Brigham \& Women's Hospital
  \And
  Northeastern
  \And
  Indiana U.
  \And
  MIT
}]

\begin{abstract}
Machine learning-supported decisions, such as ordering diagnostic tests or determining preventive custody, often require converting probabilistic forecasts into binary classifications. We adopt a consequentialist perspective from decision theory to argue that evaluation methods should prioritize forecast quality across thresholds and base rates. This motivates the use of proper scoring rules such as the Brier score and log loss. However, our empirical review of practices at major ML venues (ICML, FAccT, CHIL) reveals a dominant reliance on top-$K$ metrics or fixed-threshold evaluations.
To bridge this disconnect, we introduce a decision-theoretic framework that maps evaluation metrics to their appropriate use cases, accompanied by a practical Python package, \texttt{briertools}, which lowers the barrier to applying proper scoring rules in practice. Methodologically, we derive and implement a clipped Brier score variant that avoids full integration and better reflects bounded, interpretable threshold ranges. Theoretically, we reconcile the Brier score with decision curve analysis, directly addressing the critique of \citet{vickers17} regarding the clinical utility of proper scoring rules.
\end{abstract}

\input{sections/introduction.tex}
\input{sections/relatedwork.tex}
\input{sections/llm.tex}
\input{sections/background.tex}
\input{sections/AUC.tex}
\input{sections/brier.tex}
\input{sections/dca.tex}

\input{sections/bounded.tex}
\input{sections/dca2.tex}
\input{sections/truncated_auc.tex}
\input{sections/briertools.tex}
\input{sections/cancer.tex}
\input{sections/calibration.tex}
\input{sections/discussion.tex}

\section{Acknowledgements}
We thank the participating women, mammography facilities, and radiologists for the data they have provided. You can learn more about the BCSC at: http://www.bcsc-research.org/.

\bibliographystyle{abbrvnat}
\bibliography{references}

\section*{Checklist}

\begin{enumerate}

  \item For all models and algorithms presented, check if you include:
  \begin{enumerate}
    \item A clear description of the mathematical setting, assumptions, algorithm, and/or model. [Yes]
    \item An analysis of the properties and complexity (time, space, sample size) of any algorithm. [Yes]
    \item (Optional) Anonymized source code, with specification of all dependencies, including external libraries. [Yes]
  \end{enumerate}

  \item For any theoretical claim, check if you include:
  \begin{enumerate}
    \item Statements of the full set of assumptions of all theoretical results. [Yes]
    \item Complete proofs of all theoretical results. [Yes]
    \item Clear explanations of any assumptions. [Yes]     
  \end{enumerate}

  \item For all figures and tables that present empirical results, check if you include:
  \begin{enumerate}
    \item The code, data, and instructions needed to reproduce the main experimental results (either in the supplemental material or as a URL). [Yes]
    \item All the training details (e.g., data splits, hyperparameters, how they were chosen). [Yes]
    \item A clear definition of the specific measure or statistics and error bars (e.g., with respect to the random seed after running experiments multiple times). [Yes]
    \item A description of the computing infrastructure used. (e.g., type of GPUs, internal cluster, or cloud provider). [Yes]
  \end{enumerate}

  \item If you are using existing assets (e.g., code, data, models) or curating/releasing new assets, check if you include:
  \begin{enumerate}
    \item Citations of the creator If your work uses existing assets. [Yes]
    \item The license information of the assets, if applicable. [Yes]
    \item New assets either in the supplemental material or as a URL, if applicable. [Not Applicable]
    \item Information about consent from data providers/curators. [Not Applicable]
    \item Discussion of sensible content if applicable, e.g., personally identifiable information or offensive content. [Not Applicable]
  \end{enumerate}

  \item If you used crowdsourcing or conducted research with human subjects, check if you include:
  \begin{enumerate}
    \item The full text of instructions given to participants and screenshots. [Not Applicable]
    \item Descriptions of potential participant risks, with links to Institutional Review Board (IRB) approvals if applicable. [Not Applicable]
    \item The estimated hourly wage paid to participants and the total amount spent on participant compensation. [Not Applicable]
  \end{enumerate}

\end{enumerate}

\clearpage
\appendix
\thispagestyle{empty}

\onecolumn
\aistatstitle{Appendix: Proofs and Further Details}
\input{apdx/relatedwork.tex}
\input{apdx/angstrom.tex}
\input{apdx/regret}
\input{apdx/slice}
\input{apdx/dca}
\input{apdx/costevaluation}
\input{apdx/pr}
\input{apdx/skillscore.tex}
\input{apdx/llm}

\end{document}

%% file: sections/introduction.tex
%
\section{INTRODUCTION}
We study a setting in which a binary classifier $\f(\cdot\,; \TAU) \hspace{-.2em}: \hspace{-.1em} \mathcal{X} \hspace{-.2em} \rightarrow \hspace{-.2em}  \{0, 1\}$ is developed to map an input $x \in \mathcal{X}$ to a binary decision.
Such classifiers are foundational to decision-making tasks across domains, from healthcare to criminal justice, where outcomes depend on accurate binary choices.
The decision is typically made by comparing a score $s(x) \in \mathbb{R}$, such as a probability or a logit, to a threshold $\TAU \in \mathbb{R}$:
$$
\f(x; \TAU) = \begin{cases}
1 & \text{if } s(x) \geq \TAU \\
0 & \text{if } s(x) < \TAU.
\end{cases}
$$
The threshold $\TAU$ is a parameter that can be adjusted to control the tradeoff between false positives and false negatives, reflecting the specific priorities or constraints of a given application.
For example, consider a scenario in which a classifier is used to make (a) judicial decisions, such as who to sentence, or (b) medical decisions, such as recommending treatments for diagnosed conditions.
Which threshold should be chosen and how should the resulting classifiers be evaluated?
We advocate for a {\em consequentialist view} of classifier evaluation:  pragmatic classifier evaluation should model the real-world impacts of using that classifier to make concrete decisions.

To this end, we introduce a {\em value} function, $\V(\f(x; \TAU), y)$, which assigns a value to each prediction given the true label $y$ and the classifier's decision $\f(x; \TAU)$.
The overall performance of a classifier is given by its expected value over a distribution $(\mathcal{D}$: ${\mathbb{E}}_{(x,y)\sim \mathcal{D}} [ \V(\f(x;\TAU),y)].)$

We identify two key factors that determine appropriate evaluation approaches:
\begin{enumerate}
  \item \textbf{Instance Coupling.} whether decision costs are additive, allowing independent decisions across instances (Indep.), or there is a fixed quota to act on positive predictions (Top-K decisions); and
  \item \textbf{Threshold Specificity.} whether the decision threshold is known exactly or approximately at model selection time (regardless of what will later be known at deployment time).
\end{enumerate}
Table~\ref{tab:metric_taxonomy} illustrates how different evaluation metrics align with these settings. We develop a taxonomy for the decision problems for which classifiers are typically used, and we describe which commonly used metrics represent a decision-theoretic evaluation of those problems in potentially uncertain decision contexts. This taxonomy makes it simple to choose metrics that are consistent with deployment scenarios.
\begin{table}[ht]
  \centering
  \renewcommand{\arraystretch}{1.1}
  \setlength{\arrayrulewidth}{0.6pt}
  \setlength{\tabcolsep}{4pt}
  \begin{tabular}{c@{\hspace{4pt}}cc|>{\columncolor[HTML]{e6f2ff}}c|>{\columncolor[HTML]{f7fbff}}c}
    & & \multicolumn{1}{c}{} & \multicolumn{2}{c}{\textbf{Instance Coupling}} \\
    & & \multicolumn{1}{c}{} & \textbf{Independent} & \textbf{Top-K} \\
    \noalign{\global\cmidrulewidth=0.5pt \global\aboverulesep=0pt \global\belowrulesep=0pt}\cmidrule{4-5}\noalign{\global\cmidrulewidth=\lightrulewidth \global\aboverulesep=0.4ex \global\belowrulesep=0.65ex}
    \multirow{2}{*}[+0.6em]{\rotatebox{90}{\textbf{Threshold}}}
    & \multirow{2}{*}[+0.6em]{\rotatebox{90}{\textbf{Specificity}}}
    & \cellcolor[HTML]{e6f2ff} \textbf{Mixed $\tau$}
    & \cellcolor[HTML]{b3d8f2}$\begin{array}{c} \text{Brier} \\ \text{Log Loss}\end{array}$
    & \cellcolor[HTML]{e6f2ff}$\begin{array}{c} \text{AUC-ROC} \\ \text{AUC-PR} \end{array}$\\
    \noalign{\global\cmidrulewidth=0.5pt \global\aboverulesep=0pt \global\belowrulesep=0pt}\cmidrule{3-5}\noalign{\global\cmidrulewidth=\lightrulewidth \global\aboverulesep=0.4ex \global\belowrulesep=0.65ex}
    & & \cellcolor[HTML]{f7fbff} \textbf{Fixed $\tau$}
    & \cellcolor[HTML]{e6f2ff}$\begin{array}{c} \text{Net Benefit} \\ \text{Accuracy} \end{array}$
    & \cellcolor[HTML]{ffffff}$\begin{array}{c} \text{Precision@K} \\ \text{Recall@K} \end{array}$ \\
  \end{tabular}
  \caption{\small\textbf{Metric Selection Recommendations.}
First, identify whether decisions can be made independently or there is a fixed budget of positive labels (column).
Then, identify whether the threshold is uncertain or precisely known (row).
The dark blue cell is the most common decision context; the light blue cells' metrics are frequently used despite worse metric-context alignment (note especially Accuracy, an unrealistic special case of Net Benefit that is the most common metric overall).  The white cell is included for completeness.
}
  \label{tab:metric_taxonomy}
\end{table}

Despite pervasive threshold uncertainty in real-world ML applications, such as healthcare and the criminal legal system, evaluation practices remain misaligned with deployment realities. Evaluations based on independent decisions typically assume fixed thresholds, while evaluations that are based on threshold mixtures typically assume dependent decisions.
Our analysis of three major ML conferences (ICML, FAccT, CHIL) reveals a consistent preference for metrics designed for fixed or top-K decisions, which do not match common deployment settings where thresholds are uncertain and decisions are independent.

\subsection{Contributions}

\paragraph{Core theoretical contributions.}

Our main new result is a \textbf{bounded-threshold extension of proper scoring rules}.
For independent decisions with uncertainty over a bounded interval of cost ratios, we derive bounded versions of the Brier score and log loss that average regret only over that interval rather than over the full unit interval.
This fills the gap between point-threshold metrics such as net benefit and full-interval proper scoring rules.

Our second contribution is a \textbf{reconciliation of proper scoring rules with Decision Curve Analysis} in a common regret-based framework, which clarifies when each applies and shows how bounded-threshold scores address the critique of \citet{vickers17} that standard Brier evaluation averages over implausible thresholds.

As an appendix-level result, we also derive an interpretation of Concordant Partial AUC in the same framework, showing it implicitly averages over a score-induced threshold distribution.

\paragraph{Practice-facing contributions.}
We provide a decision-context taxonomy, supported by an LLM-assisted survey of current evaluation practice, for matching binary-classification metrics to deployment settings.
To support this use in practice, we release \texttt{briertools}, which implements the proposed metrics and associated visualizations, and we illustrate the framework on a breast-cancer treatment case study.

%% file: sections/relatedwork.tex
%
\section{Preliminaries: Current Practice and Established Theory}
This section introduces a consequentialist approach that matches metrics to implicit decision problems.
We use this approach to summarize known results about optimal thresholding under calibration, accuracy as a special case of net benefit, regret, and the structure of proper scoring rules.
We then survey the usage of these metric families in the literature.

\subsection{Related work}
\paragraph{Dependent Decisions.}
ROC/AUC was first adopted in psychophysics and radiology; in ML, it evaluates ranking quality under varying thresholds. See Appendix \ref{apdx:relatedwork} for more details.
There have been consistent critiques of AUC's lack of calibration information \cite{calibration19, aclu23}.
\citet{hand09} and \citet{brieraucrank12} showed that AUC-ROC can be interpreted as a cost-weighted average regret, particularly under calibrated or quantile-based forecasts.
Several attempts have been made to evaluate area under only portions of the ROC curve \citep{mcclish89, balancedaccuracy23}, but these have not produced clear decision-theoretic interpretations.

\vspace{-1.0em}
\paragraph{Independent Decisions.}
\citet{vickers06, vickers08} and \citet{vickers17} introduced decision curve analysis (DCA) as a threshold-restricted net benefit visualization, arguing it offers greater clinical relevance than Brier-based aggregation across all thresholds.
Recent work has examined the decomposition of Brier score and log loss into calibration and discrimination components \citep{shen05, siegert17, tryptich24}, providing implementation and visualization guidance.

\citet{shuford66} first characterize proper scoring rules by an integral formulation, and \citet{savage71} characterize the variable of integration as the rate of substitution.  \citet{schervish89} realizes this is an integral over binary decision problems at different cost ratios, and \citet{shen05} interprets the weight as describing which thresholds are important to the evaluator.  \citet{proper07} then connects this integral representation to convexity properties of proper scoring rules and extends it beyond the binary case.

\citet{brieraucrank12} then studies the behavior of this integral under different threshold decision rules.
\citet{pfohl22} conclude overall net benefit is maximized with empirical risk minimization and recalibration, and that picking optimal thresholds is equivalent to this post-hoc recalibration.

\vspace{-1.0em}
\paragraph{Surrogate Loss Consistency}
The requirement that evaluation metrics align with the downstream decision problem is conceptually related to the notion of {\em surrogate loss consistency} in statistical learning. In that setting, the question is whether the surrogate loss used for training leads to a model that is optimal for the true target loss. This principle was introduced in inference by \citet{fisher22} and developed in a modern machine learning context by \citet{bartlett06} and \citet{zhang04}. The focus there, however, is on the choice of training loss and its consistency with the parameter being estimated. Our setting is different: we are concerned not with training-time surrogate losses, but with deployment-time evaluation metrics. Uncertainty in our case arises from decision contexts and thresholds, rather than from prediction values or parameter estimation.
%

%% file: sections/llm.tex
%
\subsection{Common Metric Choices}
\label{subsec:motivating_experiment}
We analyze evaluation metrics used in papers from ICML 2024, FAccT 2024, and CHIL 2024 using an LLM-assisted review (see Appendix~\ref{apdx:llm} for more details on our analysis).
Accuracy was the most common metric at ICML and FAccT ($>50\%$), followed by AUC-ROC; CHIL favored AUC-ROC, with AUC-PR also notable.
Proper scoring rules (e.g. Brier score, log loss) were rarely used ($<15\%$ and $<5\%$, respectively).
These findings, summarized in Figure~\ref{fig:litreview}, confirm the dominance of accuracy and AUC-ROC in practice.
This paper addresses this gap by clarifying when Brier scores and log losses are appropriate; further, it provides theoretical justification and tools to support their adoption.
\begin{figure}[ht!]
\begin{center}
\includegraphics[width=\columnwidth]{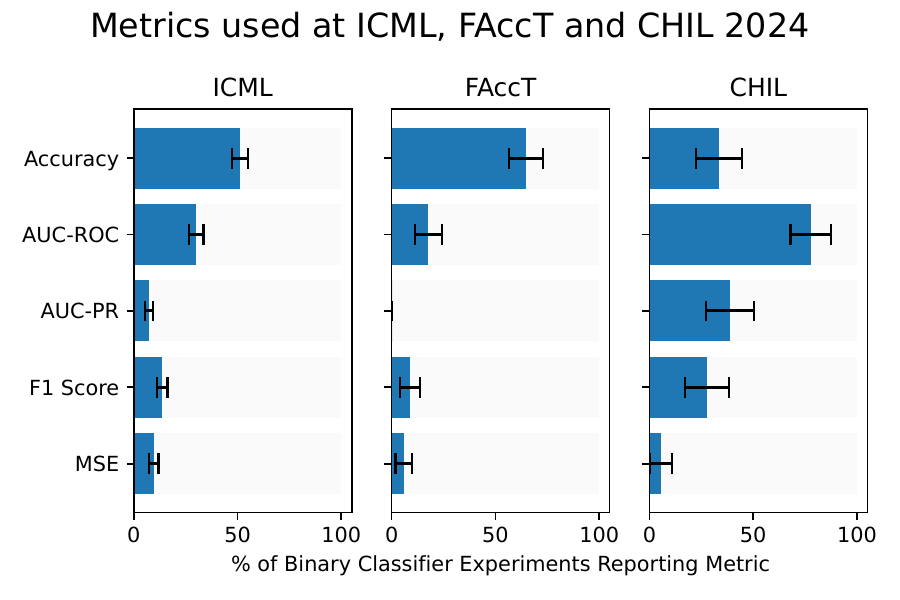}
\caption{
\small Claude 3.5 Haiku was used to analyze 2,610 papers from three major 2024 conferences.
Each plot summarizes the evaluation metrics used for binary classifiers.
Accuracy dominates outside healthcare, while AUC-ROC is more prevalent within healthcare domains.
Error bars come from binomial confidence intervals.
}
\label{fig:litreview}
\end{center}
\vskip -0.2in
\end{figure}

%% file: sections/background.tex
%

\subsection{Consequentialist Formalism}
Our goal is to evaluate binary classifiers in a way that directly reflects their real-world decision consequences.
A natural abstraction is {\bf expected regret}, which we define as \emph{the excess cost of using a model compared to the best possible decision rule under a given cost structure}.

\paragraph{Cost model} We adopt the cost model introduced by \citet{angstrom22}, where perfect prediction defines a zero-cost baseline, predicted positives incur an immediate cost $C$, and false negatives incur a downstream loss $L$.
Without loss of generality, we normalize $L = 1$ and define the relative cost as $c = C/L$.
See Appendix~\ref{apdx:angstrom} for more details about these choices and their implications.

\begin{table}[ht]
    \centering
    \begin{tabular}{c|cc}
        $\V(y, a)$ & $a = 0$ & $a = 1$ \\
        \hline
        $y = 0$ & 0 \textcolor{lightgray}{(True Neg)} & $c$ \textcolor{lightgray}{(False Pos)} \\
        $y = 1$ & $1 - c$ \textcolor{lightgray}{(False Neg)} & 0 \textcolor{lightgray}{(True Pos)}
    \end{tabular}
\end{table}
This structure highlights that all evaluation reduces to how well a classifier navigates the cost tradeoff encoded by $c$.

\paragraph{Regret} Let $\pi = P(y=1)$ denote prevalence, $F_0(\tau)$ the CDF of negative scores, and $F_1(\tau)$ the CDF of positive scores. We define regret as
\[\begin{aligned}
    \R(\f,\pi,&c,\TAU,\mathcal{D}) = \underset{(x,y)\sim \mathcal{D}}{\mathbb{E}} \Bigl[ \V(\f(x;\TAU),y)\Bigr]\\
  &= c \cdot (1-\pi) \cdot (1-F_0(\TAU))
    + (1-c) \cdot \pi \cdot F_1(\TAU).
\end{aligned}\]
This quantifies the expected loss of a classifier at threshold $\tau$ relative to the idealized zero-cost baseline \citep{brieraucrank12}.

\begin{theorem}[Optimal Threshold]
\label{thm:optimal_threshold_body}
\label{thm:opt_threshold}
Given a calibrated model, the optimal threshold is the cost:
\[
\arg\min_{\TAU} \R(\f,\pi,c,\TAU,\mathcal{D}) = c.
\]
\end{theorem}
See Appendix \ref{thm:optimal_threshold} for a brief proof; see also \citet{brieraucrank12}.
\paragraph{Assumptions}
To isolate the role of {\bf cost} in shaping evaluation, we adopt three simplifying assumptions:
\begin{itemize}
    \item {\bf Fixed prevalence} ($\pi$ does not shift between training and deployment). This removes confounding from label shift and allows us to focus purely on cost-driven evaluation.
    \item {\bf Score as probability}
    The decision-maker treats the model's output $s(x)$ as a probability, that is, acts as though $P(Y=1 | s(X)=p) = p$.  This allows us to model the decision-maker's threshold-setting behavior.
    \item {\bf Optimal thresholding}
    ($\tau^\ast = c$ is always chosen).
    This assumes the decision-maker makes the optimal choice given their beliefs and the cost ratio.
    This lets us express regret in its minimal form:
     \begin{align*}
  \Rmin(c) &  =  c\cdot(1-\pi)\cdot(1-F_0(c))
  + (1-c)\cdot  \pi\cdot  F_1(c).
        \end{align*}
\end{itemize}
These assumptions restrict attention to the core decision problem: how well does a classifier balance false positives and false negatives given a cost ratio?
By abstracting away dataset shift and thresholding complications, we can evaluate metrics solely on their ability to reflect cost-sensitive decision quality.

The second assumption is about the behavior of the decision-maker, not the match between the model and the real world (see the concept of small worlds in \cite{savage54}).
If in fact the model is poorly calibrated, \citet{brieraucrank12} show theoretically that picking an optimal threshold is tantamount to recalibrating the model and then using the formula from the third assumption.  \citet{pfohl22} verify this empirically.

In what follows, we show that many standard metrics, like accuracy and AUC-ROC, can be reinterpreted as expectations of regret over particular cost distributions.
We clarify which metrics align with which kinds of decision problems.

\subsection{Independent Decisions with Fixed Thresholds: a Consequentialist View of Accuracy}

Accuracy is the most commonly used metric for evaluating binary classifiers, as it offers a simple measure of correctness that remains the default in many settings \cite{rocaccuracy05}. Formally:

\begin{definition}[Accuracy]
Given data $\{(x_i, y_i)\}_{i=1}^n$ with $y_i \in \{0,1\}$, and a binary classifier $\f(x; \TAU)$ thresholded at $\TAU$, accuracy is defined as:
\[
\text{Accuracy}(\f, \mathcal{D}) \triangleq \tfrac{1}{n} \sum_{i=1}^n \mathbb{I}(\f(x_i; \TAU) = y_i).
\]
\end{definition}
\vspace{-1em}
Accuracy corresponds to regret minimization when misclassification costs are equal:

\begin{proposition}
\label{prop:accuracy_regret}
Let $\TAU$ denote a (possibly suboptimal) threshold. Then,
\[
\text{Accuracy}(\f, \mathcal{D}) = 1 - 2 \cdot \R(\f, \pi, c = 1/2, \TAU, \mathcal{D}).
\]
\end{proposition}

This equivalence, proved in Appendix~\ref{thm:accuracy_regret} and established by \citet{brieraucrank12}, underscores a key limitation: accuracy implicitly assumes all errors carry equal cost. In many domains, this is neither justified nor appropriate. In criminal sentencing, optimizing for accuracy equates wrongful imprisonment with wrongful release, an assumption at odds with most legal and ethical frameworks.
In prostate cancer screening, false negatives may lead to death, while false positives can result in unnecessary treatment that causes erectile dysfunction.
It is natural to model the cost of death (normalized to 1) as equal across all patients.
However, the implied cost ratio $c=1/2$ (as if erectile dysfunction were half as bad as death) both oversimplifies outcomes and disregards heterogeneous patient preferences.
Accuracy is therefore only meaningful when error costs are balanced, prevalence is stable, and tradeoffs are collectively agreed upon.
These conditions are rarely satisfied in practice.

%% file: sections/AUC.tex
\subsection{Dependent Decisions (top K) with Mixtures of Thresholds: A Consequentialist View of AUC-ROC}
In resource-constrained settings, such as allocating ICU beds, allocation decisions are coupled by the fixed budget of K predicted positives.
When $K$ is known exactly at model selection time, Net Benefit@$K$ is the theoretically appropriate metric; however, we show in Appendix~\ref{apdx:fixed_k} that it induces the same model ranking as the more popular Precision@$K$ or Recall@$K$ at fixed $N$, $K$, and $c$.

When $K$ is not known at model selection time or varies across deployments, the decision maker faces a mixture over possible thresholds.
AUC-ROC reflects this mixture.
When scores are calibrated, AUC admits a regret interpretation due to \citet{hand09}:
\[
\mathrm{AUC\text{-}ROC}(\f)
\;=\;
1 - \frac{1}{2\pi(1-\pi)}\,
\mathbb{E}_{(x,y)\sim\mathcal{D}}\!\left[\Rmin\big(s(x)\big)\right].
\]
That is, it averages $\TAU^\ast$-regret using the \emph{model’s score distribution} as the threshold prior (derivation in Thm.~\ref{thm:hand} in Ap.~\ref{apdx:auc}).

The limitation is that model scores are trained to estimate associations with outcomes rather than the costs of decisions. As \citet{hand09} notes, this means the model itself determines the relative importance of false positives and false negatives, effectively assigning how costly it is to miss a cancer diagnosis or how acceptable it is to release a guilty person.

%% file: sections/brier.tex
%
\subsection{Independent Decisions with Mixtures of Thresholds: a Consequentialist View of Brier Scores}
While accuracy is widely used as an evaluation metric, it is rarely directly optimized.
Instead, practitioners generally use surrogate losses such as squared error and log loss (also known as cross-entropy), largely based on their differentiability.

However, these functions can also be used after training to evaluate the performance of a classifier on a given dataset.
This represents a use case that is distinctly different from surrogate loss minimization.
Decades of research in the forecasting community have demonstrated that these loss functions have a deeper interpretation: they represent distinct notions of average regret, each corresponding to different assumptions about uncertainty and decision-making (see \ref{thm:brier} and \ref{thm:ll}, which are derived in \cite{shuford66}, with interpretation clarified in \cite{schervish89}).
From a consequentialist perspective, these tractable, familiar methods are not being used to their full potential as evaluation measures.

\begin{theorem}[Brier Score as Uniform Mixture of Regret]
\label{thm:brier}
Let \(\f : \mathcal{X} \to [0,1]\) be a probabilistic classifier with score function \(s(x)\), and let \(\mathcal{D}\) be a distribution over  \((x, y) \in \mathcal{X} \times \{0,1\}\). Then the Brier score of \(\f\) is  the mean squared error between the predicted probabilities and true labels:
\[
\text{\em BS}(\f, \mathcal{D}) \triangleq \mathbb{E}_{(x,y) \sim \mathcal{D}} \left[ (y - s(x))^2 \right].
\]
Moreover, this is equivalent to the expected minimum regret over all cost ratios \(c \in [0,1]\), where regret is computed with optimal thresholding:
\[
\text{\em BS}(\f, \mathcal{D}) = \mathbb{E}_{c \sim \text{Uniform}[0,1]} \left[ \Rmin(c) \right].
\]
\end{theorem}
The result is known from \citet{shuford66}, but a more general proof appears in Appendix~\ref{apdx:bounded_logloss}, where this unbounded version arises as a special case.
The following Theorem~\ref{thm:ll} establishes that unlike the Brier score, which weights regret uniformly across thresholds, log loss emphasizes extreme cost ratios via the weight $\frac{1}{c(1 - c)}$.

\begin{theorem}[Log Loss as a Weighted Average of Regret]
\label{thm:ll}
Let $\f: \mathcal{X} \to [0,1]$ be a probabilistic classifier with score $s(x)$, and let $\mathcal{D}$ be a distribution over $(x, y) \in \mathcal{X} \times \{0,1\}$. Then:
\begin{align*}
    \text{\em LL}(\f, \mathcal{D}) &= \mathbb{E}_{(x,y)\sim \mathcal{D}} \left[ -\log \left( s(x)^y (1 - s(x))^{1 - y} \right) \right]\\& =
\int_0^1 \frac{\Rmin(c)}{c(1 - c)} \, dc = \int_{-\infty}^{\infty} \hspace{-.5em} \Rmin\left( \frac{1}{1 + e^{-\ell}} \right) d\ell.
\end{align*}

\end{theorem}

Like the Brier score, log loss integrates regret uniformly over log-odds of cost ratios, assigning more weight to rare but high-consequence decisions.
This makes log loss more sensitive to tail risks, which may be desirable when one type of error carries disproportionate cost.
See Figure~\ref{fig:prior_weights} in Appendix~\ref{apdx:threshold_distributions} for a visualization.

In practice, even though proper scoring rules such as log loss are widely used during training, final model selection often reverts to fixed-threshold metrics.
To address this, we introduce bounded-threshold metrics that restrict evaluation to a plausible range of costs without requiring exact specification at the time of model selection.

%% file: sections/dca.tex
%
\section{Novel Results: Bounded Threshold Scoring Rules and Their Relation to DCA}
Must scoring rules represent regret averaged across all possible thresholds, or can they incorporate expert-level judgment about relevant decision contexts directly?
To investigate this question, we analyze Decision Curve Analysis (DCA) \citep{vickers06,vickers19dca}, a decision-theoretic approach from medical informatics that formed the basis of a prominent critique of the Brier score.
\citet{vickers17} argued that averaging over the full threshold range assigns weight to cost ratios no clinician would consider plausible.

We show that DCA can itself be understood as a regret-based evaluation approach, in the same family as the Brier score and log loss.
While the specific formulation of DCA makes it ill-suited to averaging across thresholds, the critique highlights the importance of restricting evaluation to clinically relevant cost ratios.
Accordingly, we introduce bounded-threshold versions of the Brier score and log loss that \emph{do} support averaging regret over restricted intervals of decision-context uncertainty.

\subsection{Decision Curve Analysis}

DCA is built around the notion of {\em net benefit}, the quantity plotted against the threshold in DCA, which we show to be equivalent to regret up to constant terms.

\begin{definition}[Net Benefit (DCA)]
As defined by \citet{vickers19dca}, the net benefit at decision threshold
$\TAU \in (0,1)$ is given by:
\[
  \text{ NB} (\TAU) = (1-F_1(\TAU)) \cdot \pi - (1-F_0(\TAU)) \cdot (1 - \pi) \cdot \tfrac{\TAU}{1-\TAU}.
\]
\end{definition}

\begin{theorem}[Net Benefit as a function of regret]
\label{thm:net_benefit_h}
\label{thm:nb_regret}
Let $\pi$ denote the prevalence of the positive class, which is the maximum achievable benefit under perfect classification. The net benefit at threshold $c$ is related to the regret as follows:
\[
  \text{\em NB}(c) = \pi - \tfrac{1}{1-c}\Rmin(c).
\]
\end{theorem}

\citet{vickers17} argue that the Brier score is inadequate for clinical settings where only a narrow range of decision thresholds is relevant (e.g. determining the need for a lymph node biopsy).
Comparing the unrestricted Brier score to net benefit at fixed thresholds (e.g., 5\%, 10\%, 20\%), they conclude that net benefit better captures clinical priorities.
However, once Brier score is understood as a weighted average of an affine transformation of net benefit across cost ratios, this critique elucidates a useful insight: the appropriate comparison for a fixed-threshold method like DCA is not to average over the full range of costs, as the Brier Score does, but to average only over the relevant interval (e.g., [5\%, 20\%]).  We now turn to this problem.

%% file: sections/bounded.tex
%
\subsection{Regret over a Bounded Range of Thresholds}
We derive a new and computationally efficient expression for expected regret when the cost ratio \(c\) is distributed uniformly over a bounded interval \([a, b] \subseteq [0, 1]\) by exploiting the duality between pointwise squared error and average regret.
This formulation not only improves numerical stability but also simplifies implementation, requiring only two evaluations of the Brier score under projection.

Selecting $[a,b]$ requires the same clinical reasoning already used in decision curve analysis: a practitioner identifies the cost ratios that are plausible for the population of interest by reasoning about what makes a threshold too low (intervention harms outweigh any realistic benefit) or too high (no patient would demand that level of certainty before acting) \citep{vickers16}. The uniform prior over this interval is then the least structured choice consistent with those bounds.  Prior approaches that require a full distributional specification, such as the Beta priors of \citet{hand09} and \citet{recentbeta24}, demand dispersion parameters that practitioners cannot easily elicit or justify. See Appendix~\ref{apdx:threshold_distributions} for a detailed comparison.

Throughout, we will use notation  \(\text{clip}_{[a,b]}(z) \triangleq \max(a, \min(b, z))\) to denote the projection of $z$ onto the interval $[a,b]$.

\begin{theorem}[Bounded Threshold Brier Score]
\label{thm:bounded_brier}
For a classifier \(\f\), the average minimal regret over cost ratios \(c \sim \text{Uniform}(a,b)\) is given by:
\begin{align*}
\underset{c \sim \text{Unif}(a,b)}{\mathbb{E}} &\Rmin(c)
=
\frac{1}{b - a} \biggl[
\underset{(x, y) \sim \mathcal{D}}{\mathbb{E}} \left( y - \text{\em clip}_{[a,b]}(s(x)) \right)^2\\
&-
\underset{(x, y) \sim \mathcal{D}}{\mathbb{E}} \left( y - \text{\em clip}_{[a,b]}(y) \right)^2\biggr].
\end{align*}
\end{theorem}
The result follows as a direct extension of the proof of Theorem \ref{thm:brier}.
Specifically, the same argument structure applies with the necessary modifications that account for the additional constraints introduced in this setting.
For a complete derivation, refer to the proof of Theorem \ref{apdx:bounded_brier} in the Appendix, where the argument is presented in full detail, including Proposition \ref{prop:propriety} which shows that the clipped version is proper but not strictly proper.

Overall, this expression offers two practical advantages.
First, it is computationally efficient, requiring only 2 Brier score evaluations, one on predictions and one on labels, after projecting onto \([a, b]\).
Second, it is interpretable, recovering the standard Brier score when \(a = 0\) and \(b = 1\), consistent with the assumption that true labels lie in \(\{0,1\}\).

\begin{theorem}[Bounded Threshold Log Loss]
\label{thm:bounded_logloss}
Let $\f$ be a probabilistic classifier with score function $s(x)$. Let $c = \frac{1}{1 + \exp(-\ell)}$ denote the cost ratio corresponding to log-odds $\ell$, and suppose $\ell$ is distributed uniformly over the interval \([\log\frac{a}{1-a},\, \log\frac{b}{1-b}]\), where \(0 < a < b < 1\). Then the expected regret over this range is given by:
\[\begin{aligned}
  &{\mathbb{E}}_{\ell\sim\text{Uniform}\bigl(
    \log\frac{a}{1-a},\; \log\frac{b}{1-b}
    \bigr)}\left[\Rmin\Big(\frac{1}{1+\exp -\ell}\Big)\right]
  \\&= \tfrac{1}{\log\frac{b}{1-b}-\log\frac{a}{1-a}}\biggl(
    \underset{(x,y)\sim\mathcal{D}}{\mathbb{E}}[\log(1-|y- \text{\em clip}_{[a,b]}(s(x))|)]
    \\&\qquad - \hspace{-.9em} \underset{(x,y)\sim\mathcal{D}}{\mathbb{E}}[\log(1-|y- \text{\em clip}_{[a,b]}(y)|)]
  \biggr).
\end{aligned}\]
\end{theorem}
Similarly, this result follows as a direct extension of the proof of Theorem \ref{thm:ll}. The argument structure remains the same, incorporating appropriate modifications to account for the additional constraints in this setting. For a complete derivation, refer to the proof of Theorem \ref{apdx:bounded_logloss} in the Appendix, where the full details are provided.
Like the bounded threshold Brier score, our bounded threshold log loss score is practical to implement: it requires only two calls to a standard log loss function with clipping applied to inputs.
Moreover, when \(a = 0\) and \(b = 1\), the second term vanishes, recovering the standard log loss.

\begin{remark}[Bounded Threshold AUC-ROC]
If we assume that a model is calibrated, then a similar idea can be applied to the AUC-ROC. The resulting metric is equivalent to an expected regret at a set of thresholds defined by the score distribution of the model \textit{in the interval} (see Theorem \ref{thm:bounded_auc} in the Appendix).
However, all of Hand’s original critiques remain. The model is still trained to estimate associations with outcomes rather than decision costs. As a result, normative judgments about cost, harm, and acceptability are deferred to the statistical model rather than specified directly.
\end{remark}

\begin{remark}[Skill Scores]
Although the bounded threshold metrics are precisely an expected regret, and can be rescaled to match other decision costs, in casual use it can be simpler to use them while rescaling so 100\% is a perfect model and 0\% is a naive baseline.
This is known as a skill score; we expand on this in detail in Appendix~\ref{apdx:skillscore}.
\end{remark}

%% file: sections/dca2.tex
\subsection{Revisiting the Brier score Critique by \citet{vickers17}}

We reproduce the original results from \citet{vickers17} (see Appendix~\ref{apdx:comparison_vickers17}) and show that bounded Brier score rankings closely track those of net benefit at 5\%. The scores diverge only when net benefit itself varies substantially across thresholds. This indicates that the main obstacle has been tooling rather than theory. Bounded scoring rules provide a principled and interpretable alternative that respects threshold constraints while aligning more closely with clinical decision-making.

A natural question then arises: why not simply average net benefit over the relevant interval?
Area-under-the-curve (AUC) aggregation is not typically practiced in DCA \citep{vickers08}, in part because it lacks a clear interpretation. By leveraging our equivalence with regret, we show how a mathematically similar procedure yields a well-defined average of net benefit over a bounded range of thresholds.

\begin{theorem}[Bdd Threshold Net Benefit]
\label{thm:bounded_dca}
Let $L(x,y) = \begin{cases}
  s(x) & \text{if } y=1 \\
  (1 - s(x)) - \ln(1-s(x)) & \text{if } y=0
\end{cases}$ be a pointwise loss. For a classifier $\f$, the integral of net benefit (NB) over the interval [a,b] is the loss for the predictions clipped to [a,b] minus the loss for the true labels clipped to [a,b]:
\begin{align*}
  \underset{c \sim \text{Unif}(a,b)}{\mathbb{E}}\text{\em NB} (c)
  &=
  \pi -
  \tfrac{1}{b-a}\biggl[
  \underset{(x,y)\sim\mathcal{D}}{\mathbb{E}}
  L( \text{\em clip}_{[a,b]}(s(x)),y)\\
  &\quad- \underset{(x,y)\sim\mathcal{D}}{\mathbb{E}} L(\text{\em clip}_{[a,b]}(y),y)
  \biggr].
\end{align*}
\end{theorem}
While mathematical equivalence resolves formal concerns, it does not address semantic limitations.
In prostate cancer screening, for instance, patients may agree on the value of survival as a primary goal but disagree on how to value life with treatment side effects. The direct benefit of treatment is therefore constant, but the net benefit of treatment varies with the threshold.
Standard DCA treats the gross benefit minus gross cost of a true positive as fixed across patients, even when their valuations of treatment burdens differ.
This is inconsistent in settings with heterogeneous preferences.
The Brier score, in contrast, fixes the direct false negative penalty while varying the overtreatment cost with the threshold.
This allows the net value of a true positive to adjust with patient preferences, yielding more coherent semantics for population-level averaging under cost heterogeneity.
These semantics can also be recovered from decision curves through axis rescaling: quadratic transformations yield the Brier score (Appendix~\ref{apdx:quad_dca}), while logarithmic transformations yield log loss (Appendix~\ref{apdx:log_dca}).
Figure~\ref{fig:dca} illustrates these relationships.

Although DCA is mathematically close to the Brier score and log loss, its formulation is inherently pointwise. Averaging over thresholds therefore lacks a clear interpretation, limiting DCA to fixed-threshold comparisons.

%% file: sections/truncated_auc.tex
%

%% file: sections/briertools.tex
%
\section{\texttt{briertools}: A PYTHON PACKAGE FOR FACILITATING THE ADOPTION OF BRIER SCORES}
\label{sec:brier}

We introduce a Python package, \texttt{briertools}, designed to fill the gap in evaluation support tools by making Brier scores and their bounded-threshold variants easier to apply in practice. The package provides utilities for computing bounded-threshold scoring metrics and for visualizing the associated regret and decision curves. It is installable via \texttt{pip} and supports common use cases with minimal overhead.

In addition to providing efficient implementations that exploit the duality between pointwise error and average regret, \texttt{briertools} includes plotting functions for regret curves and threshold-based diagnostics.
Although plotting regret against thresholds is slower and less precise than our formulas for estimating quadrature, it is valuable for debugging and exploratory analysis.
As recommended by \citet{tryptich24}, such visualizations help detect unexpected behaviors across thresholds and offer deeper insight into model performance under varying decision boundaries.

\begin{figure}[ht]
\begin{center}
\includegraphics[width=\columnwidth]{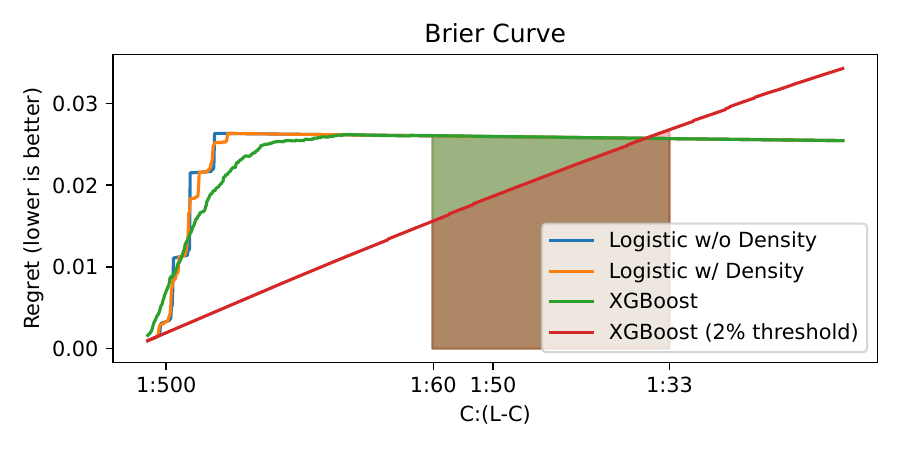}
\vskip -0.1in
\caption{
\small A comparison of breast cancer prediction models' performance over the range of commonly suggested thresholds for treatment.
}
\label{fig:tamoxifen}
\end{center}
\end{figure}

%% file: sections/cancer.tex
%
\subsection{Example: Breast Cancer Risk Prediction under Ambiguous Treatment Thresholds}

\begin{table}[ht]
\centering
\scriptsize
\begin{tabular}{lcccc}
\toprule
Metric & w/o dens. & w/ dens. & xgb & xgb (2\%) \\
\midrule
AUC & 0.601 & \textbf{0.606} & 0.574 & 0.591 \\
Brier & 0.0263 & 0.0263 & \textbf{0.0263} & 0.0415 \\
Log Loss & 0.156 & 0.156 & \textbf{0.155} & 0.193 \\
Thr. LogLoss & 0.0153 & 0.0153 & 0.0153 & \textbf{0.0123} \\
Thr. Brier & 0.000662 & 0.000662 & 0.000662 & \textbf{0.000545} \\
\bottomrule
\end{tabular}
\caption{Performance comparison of classifiers on breast cancer risk prediction. Modifying the internal decision threshold for xgboost during training decreases overall performance but improves performance in the clinically relevant range of thresholds.}
\label{tab:cancer_results}
\end{table}

Treatment guidelines for breast cancer prevention often recommend Tamoxifen when a patient’s estimated cancer risk exceeds a specified threshold.
However, there is no consensus as to what this threshold should be. Historical standards have set benefit at 1.66\% risk \cite{barlow06}, while the US Preventive Services Task Force now recommends 3\%.
A 2\% threshold is also frequently cited in the literature. Such disagreement leaves clinicians with the challenge of making treatment decisions under uncertain, contested standards.

\paragraph{Methods}
We analyze the Breast Cancer Surveillance Consortium Risk Estimation Dataset
\footnote{Licensing details can be found at \url{https://www.bcsc-research.org/index.php/datasets/rfdataset}.}
\cite{barlow06},
following prior work by \citet{yang22}.
They compared models trained with and without breast tissue density as a predictor, finding little difference in baseline Brier score.
Building on this, we additionally train
\footnote{We trained using CPU only, on an off-the-shelf laptop.}
and evaluate two XGBoost classifiers, one of which was modified to use an internal decision threshold of 2\%.
We assess all models with both traditional metrics (AUC-ROC, Brier score, log loss) and bounded-threshold scoring rules over the clinically plausible range of 1.66\%-3\%.

\paragraph{Findings}
Across the full range of thresholds, the modified XGBoost model performs worse than both the baseline XGBoost and the logistic models from \citet{yang22}. However, when evaluation is restricted to the clinically relevant interval, it outperforms all alternatives.

\paragraph{Sensitivity to Bound Selection}
The Log Loss curve (Figure~\ref{fig:tamoxifen}) supports a visual sensitivity analysis.
The bounded-threshold Log Loss for any interval $[a,b]$ is the average height of a model's curve over that interval.
In the clinically relevant region near $[1.66\%,\, 3\%]$, the XGBoost model with a 2\% internal threshold (red) has the lowest regret; this ranking is stable under moderate perturbation of the bounds.
As the interval widens toward the full $[0,1]$, the red curve's steep rise at higher cost ratios eventually dominates, which is why the global Log Loss ranks the model worst.

\paragraph{Implications}
This case study demonstrates that threshold-aware evaluation can change model selection. While global metrics penalize models that deviate from average performance, bounded-threshold scoring rules reveal which models are best suited for specific decision contexts. In domains such as medicine, where professional consensus on treatment thresholds is often absent, evaluation methods must capture performance over plausible ranges rather than relying on fixed cutoffs.

%

%

%

%

%

%

%

%

%% file: sections/calibration.tex
%
\paragraph{Commensurable Calibration and Discrimination}
A distinctive feature of \texttt{briertools} is that it allows users to evaluate calibration and discrimination on a common, commensurable scale. This is something that standard evaluation practices cannot do.

Top-\(K\) metrics, for example, evaluate only the ordering of predicted scores and are entirely insensitive to calibration. They can be used only under the assumption that more than
\(K\) individuals would benefit from a positive label. In response, it is often recommended that practitioners report both AUC-ROC and Expected Calibration Error (ECE). Yet these numbers cannot be meaningfully combined, leaving practitioners without a principled way to make joint decisions.

In contrast, proper scoring rules such as the Brier score and log loss inherently account for both discrimination and calibration \citep{shen05, tryptich24}; they also admit additive decompositions that make this distinction explicit.
We obtain the post-hoc calibrated model by applying the Pool Adjacent Violators (PAV) algorithm to the test predictions, following the procedure described in \citet{tryptich24}.
The additive decomposition then follows directly.
For the Brier Score, this is a standard bias-variance decomposition for squared error, and for log loss, the decomposition separates calibration error from irreducible uncertainty via KL-divergence between the calibrated and uncalibrated models \citep{shen05}.
\begin{figure}[ht]
\begin{center}
\includegraphics[width=\columnwidth]{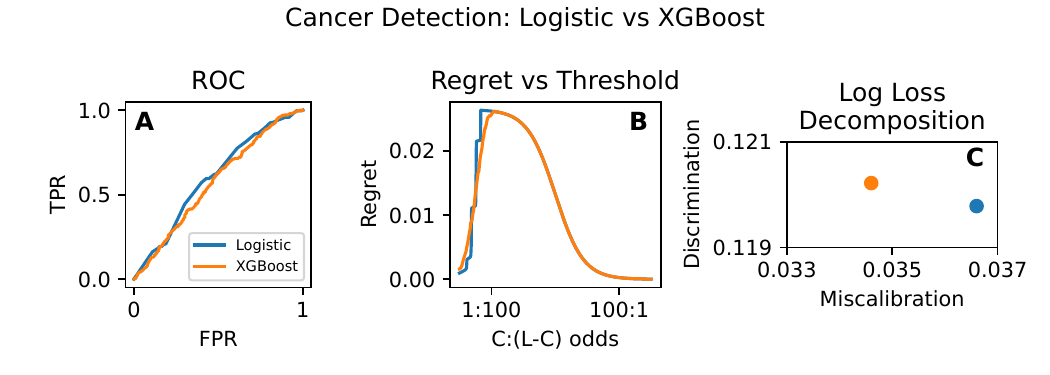}
\vskip -0.1in
\caption{
\small A. ROC plot shows XGBoost has worse discrimination than Logistic Regression models as recommended by \citet{yang22}.  B. Log Loss Curve shows XGBoost has better average regret.  C.  Decomposition reconciles the two; XGBoost has much better calibration, and only slightly worse discrimination.
}
\label{fig:rocvbrier}
\end{center}
\end{figure}

This matters for practitioners because miscalibration can strongly affect evaluation outcomes. Subgroup analyses based on top-
\(K\) metrics, for instance, may suggest misleading fairness conclusions when calibration is poor \citep{subgroupauc19}, and AUC-ROC does not reflect error rates at operational thresholds \citep{aclu23}. Figure~\ref{fig:rocvbrier} shows how log loss decomposition makes these trade-offs transparent: a model with higher AUC but poor calibration (blue) can be revealed as inferior to one with slightly weaker discrimination but much stronger calibration (orange). By surfacing this calibration gap directly, \texttt{briertools} enables practitioners to evaluate models in a way that is both rigorous and actionable.

%% file: sections/discussion.tex
%
\section{DISCUSSION}

The regret formalism reveals a unifying structure beneath Table~\ref{tab:metric_taxonomy}: every metric in the table is a weighted average of $\Rmin(c)$, and the metrics differ only in which cost ratios receive weight.
This makes the table a prescriptive guide: to choose a metric, identify the deployment setting and use the weighting that matches it.
Reading the table against our survey of 2,610 papers from ICML, FAccT, and CHIL~2024 reveals a systematic mismatch between the metrics practitioners report and the settings they occupy.

\textbf{Dependent decisions.}
When a fixed budget couples decisions across instances, ranking metrics become appropriate: Net Benefit@$K$ when the budget $K$ is known, AUC-ROC when $K$ varies across deployments.
However, when used for independent decisions, AUC-ROC implicitly weights cost ratios by the model's own score distribution~\citep{hand09}, letting a model trained to predict outcomes implicitly determine how costly each type of error is.

\textbf{Fixed thresholds.}
For independent decisions, accuracy dominates practice (${>}50\%$ of binary-classification papers at ICML and FAccT), yet it is simply net benefit at the special case $c = 1/2$ (Proposition~\ref{prop:accuracy_regret}), justified only when error costs are equal.
DCA allows asymmetric costs but remains pointwise, evaluating $\Rmin$ at a single cost ratio.

\textbf{Uncertain thresholds.}
Eliminating coupled decisions and fixed thresholds places us in the top-left cell of Table~\ref{tab:metric_taxonomy}---the setting of most real deployments, yet the least used in our survey (proper scoring rules appear in fewer than $15\%$ of papers and fewer than $5\%$ at ICML).
Standard proper scoring rules average $\Rmin(c)$ over the full unit interval, giving equal weight to irrelevant cost ratios---precisely the over-averaging that~\citet{vickers17} criticized.
Our bounded-threshold variants (Theorems~\ref{thm:bounded_brier},~\ref{thm:bounded_logloss}) resolve this by averaging only over a practitioner-specified interval $[a,b]$.
The DCA reconciliation (Theorem~\ref{thm:bounded_dca}) confirms that bounded Brier \emph{is} averaged net benefit over $[a,b]$: DCA is the right tool when the threshold is fixed, whereas bounded scores provide a principled average when the threshold is uncertain but confined to a plausible range.
The breast cancer case study shows this distinction has practical consequences: bounded evaluation reverses model rankings, selecting the model that dominates in the clinically relevant range.

These bounded scoring rules jointly reward calibration and discrimination, admit additive decompositions (Figure~\ref{fig:rocvbrier}), and encode threshold uncertainty as a simple interval rather than a full distributional specification.
As with any threshold-based evaluation, these interpretations require scores on a probability scale; this calibration assumption should be stated explicitly whenever bounded metrics are reported.

\vspace{-.5em}
\paragraph{Limitations.}
The clean semantics of averaging regret over a cost-ratio interval depend on the assumption that the gross cost of a positive label does not depend on the true class, and that the false-negative cost is invariant across cost-ratio scenarios.

\vspace{-.5em}
\paragraph{Guidance for Practitioners.}
To support adoption, we provide \texttt{briertools}, an \texttt{sklearn}-compatible package for computing and visualizing bounded-threshold scores, regret curves, and calibration decompositions.

\vspace{-.5em}
\paragraph{Future Work.}
Extending these results to deployment shift or multi-class settings would broaden their applicability.

%% file: apdx/relatedwork.tex
\section{Historical Development of Dependent vs Independent Metrics}
\label{apdx:relatedwork}
\paragraph{Dependent Decisions.}
The idea of plotting false positive rate (FPR) and true positive rate (TPR) against decision thresholds, but not against each other, originates in World War II-era work on signal detection theory \cite{north43,north63,mit51}.
The ROC plot itself emerged in post-war work on radar signal detection theory \cite{birdsall53, birdsall54} and spread to psychological signal detection theory through the work of Tanner and Swets \cite{tanner53, swetsbirdsall56}.
From there, the ROC plot was adopted in radiology, where detecting blurry tumors on X-rays was recognized as a psychophysical detection problem \cite{metz78}.
The use of the Area under Receiver Operating Characteristics Curve (AUC-ROC) began with psychophysics \cite{swets66} and was particularly embraced by the medical community \cite{metz78,hanley82}.
From there, as AUC-ROC gained traction in medical settings, \citet{spackman89} proposed its introduction to broader machine learning applications.
At a time when popular techniques like Naive Bayes and SVMs were wildly miscalibrated, being able to report pure ranking performance was useful to methods developers.
This idea was further popularized by \citet{bradley97} and extended in studies examining connections between AUC and accuracy \cite{rocaccuracy05}.
Although more widespread use of calibration techniques such as Platt scaling \citep{platt99} and isotonic regression \citep{pav55} (which is equivalent to computing the convex hull of the ROC curve \citep{pavroc07}) now facilitates reliable probability estimates, AUC-ROC remains prevalent, especially in clinical settings, due to its robustness to score miscalibration.
There have been consistent critiques of the lack of calibration information in the ROC curve \cite{calibration19}, \cite{aclu23}.

Over the years there have been a few attempts to evaluate area under only a portion of the ROC curve \citet{mcclish89}, \citet{balancedaccuracy23}, but they haven't produced a clear decision-theoretic interpretation.
\vspace{-1.0em}
\paragraph{Independent Decisions.}
The idea of maximizing expected value was popularized by \cite{neumann44}.
The link between incentive-compatible forecast metrics (e.g., Brier score \cite{brier50}, log loss \cite{good52,mccarthy56}) and expected regret was formalized by \citet{shuford66}, clarified by \citet{savage71}, and later connected to regret curves by \citet{schervish89}.
These ideas were revisited and extended through Brier Curves \citep{hand99, costcurve06, briercurve11} and Beta-distribution modeling of cost uncertainty \citep{recentbeta24}.
\citet{hand09} and \citet{brieraucrank12} showed that AUC-ROC can be interpreted as a cost-weighted average regret, especially under calibrated or quantile-based forecasts.
Separately, \citet{vickers06, vickers08} and \citet{vickers17} introduced decision curve analysis (DCA) as a threshold-restricted net benefit visualization, arguing it offers more clinical relevance than Brier-based aggregation.
Recent work has further examined the decomposability of Brier and log loss into calibration and discrimination components \citep{shen05, siegert17, tryptich24}, providing guidance on implementation and visualization.

%% file: apdx/angstrom.tex
\section{Angstrom Cost/Loss model}
\label{apdx:angstrom}
Our consequentialist framework evaluates binary decisions using (state, action, value) triplets, connecting the true class ($y$), predicted label ($a$), and outcome value ($V(y,a)$).
We use a cost structure originating with \citet{angstrom22} and used by \citet{brier55}, implicitly based on the idea that negative labels are the majority class, and that the baseline true negative should have a 0 value.
The cost/loss parameterization assumes that the failure to detect a positive case (i.e. a false negative) has a downstream cost of $L$.
It then makes the assumption that any positive prediction incurs an immediate cost $C$ (e.g., treatment cost) which does not depend on the true class.
This is usually true, and allows us to compare outcomes across different scenarios where different treatment costs are incurred.
Where it is false, the formalism is still expressive enough to fit any contingency table we might want, but comparisons across different treatment costs will no longer be valid.

\begin{table}[ht]
    \centering
    \begin{tabular}{c|cc}
        $V(y, a)$ & $a = 0$ & $a = 1$ \\
        \hline
        $y = 0$ & 0 \textcolor{lightgray}{(True Neg)} & $C$ \textcolor{lightgray}{(False Pos)} \\
        $y = 1$ & $L$ \textcolor{lightgray}{(False Neg)} & $C$ \textcolor{lightgray}{(True Pos)}
    \end{tabular}
\end{table}

We assume $C > 0$ because otherwise assigning the label $1$ strictly dominates using a classifier.
We assume $L > C$ because otherwise assigning the label $0$ strictly dominates using a classifier.
As such, the ratio $C/L$ is constrained to lie in the interval $(0,1)$.
We want clear semantics for comparing different cost ratios, so we divide thorough by $L$.
This is equivalent to assuming that the loss from false negatives is the same in all scenarios, and only the treatment cost $C$ varies.
We then borrow from online learning the idea of measuring the optimality gap rather than the absolute outcome.

\begin{table}[ht]
    \centering
    \begin{tabular}{c|cc}
        $V(y, a)$ & $a = 0$ & $a = 1$ \\
        \hline
        $y = 0$ & 0 \textcolor{lightgray}{(True Neg)} & $c$ \textcolor{lightgray}{(False Pos)} \\
        $y = 1$ & $1 - c$ \textcolor{lightgray}{(False Neg)} & 0 \textcolor{lightgray}{(True Pos)}
    \end{tabular}
\end{table}

Note that in an online setting, regret refers to an optimality gap to the best possible classifier training algorithm.
Our setting is offline, so we borrow the terminology to refer to the optimality gap with the best possible classifier.

%% file: apdx/regret.tex
\section{Regret}
\label{apdx:checklist_regret}

\begin{theorem}[Optimal Threshold]
\label{thm:optimal_threshold}
\[
  \arg\min_{\TAU} \R(\f,\pi,\TAU,c,\mathcal{D}) = c
\]
\begin{proof}
We find the stationary points as follows:
\begin{align*}
\R(\f, \pi, \TAU, c, \mathcal{D})
&= c \cdot (1 - \pi) \cdot (1 - F_0(\TAU))
\;+\; (1 - c) \cdot \pi \cdot F_1(\TAU)
\\[0.5em]
0 &= \frac{\partial \R(\f, \pi, \TAU, c, \mathcal{D})}{\partial \TAU}
\\
&= -c (1 - \pi) \cdot f_0(\TAU) + (1 - c) \pi \cdot f_1(\TAU)
\end{align*}
using the identity $\frac{d}{d\TAU} F_0(\TAU) = f_0(\TAU)$ and$\frac{d}{d\TAU} F_1(\TAU) = f_1(\TAU)$.  This gives the condition:
\[
c (1 - \pi) f_0(\TAU) = (1 - c) \pi f_1(\TAU)
\]
Rewriting this in terms of conditional probabilities:
\[
\frac{\pi f_1(\TAU)}{\pi f_1(\TAU) + (1 - \pi) f_0(\TAU)} = c
\quad \Rightarrow \quad
c =  \frac{P(y= 1, s(x)=\TAU)}{P(s(x)=\TAU)} = P(y = 1 \mid s(x) = \TAU)
\]

\end{proof}
This will be a minimum if we have convexity, so that \[
    \frac{\partial}{\partial \TAU} P(y= 1 \mid s(x)=\TAU) > 0.
\]
If the scoring function \(s(x)\) is calibrated, then: \[
    P(y= 1 \mid s(x)=\TAU) = \TAU,
\] which gives us convexity and therefore,
\[c = \TAU.\]
\end{theorem}

\begin{theorem}[Accuracy as a function of Regret]
\label{thm:accuracy_regret}
\[\text{Accuracy}(\f, \mathcal{D}) =  1 - 2\cdot \R(\f,\pi, c = 1/2, \TAU, \mathcal{D})\]
\begin{proof}
  \begin{align*}
  \text{Accuracy}(\f, \mathcal{D}) &\triangleq \frac{1}{n}\sum_{i=1}^{n}\mathbb{I}\{\f(x_i;\TAU) = y_i\}
  \\&= P\bigl(\f(x;\TAU)=y\bigr)
  \\&= P\bigl(\f(x;\TAU)=0,\, y=0\bigr) \;+\; P\bigl(\f(x;\TAU)=1,\, y=1\bigr)
  \\&= P(y=0)\,P\bigl(\f(x;\TAU)=0\mid y=0\bigr)
  \;+\; P(y=1)\,P\bigl(\f(x;\TAU)=1\mid y=1\bigr)
  \\&= (1-\pi)\,P(s(x)<\TAU\mid y=0) \;+\; \pi\,P(s(x)\ge\TAU\mid y=1)
  \\&= (1-\pi)\,F_0(\TAU) \;+\; \pi\Bigl(1-F_1(\TAU)\Bigr),
  \\&= 1 - \Bigl((1-\pi)\Bigl(1-F_0(\TAU)\Bigr) \;+\; \pi\,F_1(\TAU)\Bigr)
  \\&= 1 - 2 \biggl(\tfrac{1}{2}\Bigl((1-\pi)\Bigl(1-F_0(\TAU)\Bigr) \;+\; \pi\,F_1(\TAU)\Bigr)\biggr)
  \\&= 1 - 2\;\R(\f,\pi,c=\frac{1}{2},\TAU,\mathcal{D})
  \end{align*}
\end{proof}
\end{theorem}

%% file: apdx/slice.tex
%
\section{Appendix: Bounded Threshold Mixtures}
\label{apdx:checklist_brier}
The overall plan of this proof is to first use integration by parts to prove an equivalence between pointwise loss functions integrated over the distribution of data, and weighted $\ell^0$ loss functions integrated over an interval of costs.

\subsection{Lemmas}
\begin{lemma}[Positive Class]
\label{lem:pos_brier}
Let $0 < a < b < 1$, and let $L(x)$ be a pointwise loss function for the positive class.
\begin{align*}
  \int_{s=0}^{s=1}
  \Bigl(
    L(\max(a,\min(b,s)))
    - L(\max(a,\min(b,1)))
  \Bigr)
  \; dF_1(s)
  &=\int_{c=a}^{c=b}
  -\frac{dL(c)}{dc} \; F_1(c) \; dc
\end{align*}
The proof will simply be integration by parts, with some careful handling of the limits of integration.
\begin{proof}
\begin{align*}
  &\int_{s=0}^{s=1}
  \Bigl(
    L(\max(a,\min(b,s))) - L(\max(a,\min(b,1)))
  \Bigr)
  \; dF_1(s)
  \\&= \int_{s=0}^{s=1}
  \Bigl(
    L(\max(a,\min(b,s))) - L(b)
  \Bigr)
  \; dF_1(s)
  \\&= \int_{s=0}^{s=b}
  \Bigl(
    L(\max(a,s)) - L(b)
  \Bigr)
  \; dF_1(s)
  \\&= \int_{s=a}^{s=1}
  \Bigl(
    \int_{c=\max(a,s)}^{c=b}
    -\frac{dL(c)}{dc} \; dc
  \Bigr)
  \; dF_1(s)
  \\&= \int_{c=a}^{c=b}
  \Bigl(
    \int_{s=0}^{s=c}
    \; dF_1(s)
  \Bigr)
  -\frac{dL(c)}{dc} \; dc
  \\&= \int_{c=a}^{c=b}
  -\frac{dL(c)}{dc}
  \Bigl(
    F_1(c) - F_1(0)
  \Bigr)
  \; dc
  \\&= \int_{c=a}^{c=b}
  -\frac{dL(c)}{dc}
  \;F_1(c)
  \; dc
\end{align*}
\end{proof}
\end{lemma}

\begin{lemma}[Negative Class]
\label{lem:neg_brier}
Let $0 < a < b < 1$, and let $L(x)$ be a pointwise loss function for the negative class.
\begin{align*}
  \int_{s=0}^{s=1}
  \Bigl(
    L(\max(a,\min(b,s))) - L(\max(a, \min(b,0)))
  \Bigr)
  \; dF_0(s)
  &=\int_{c=a}^{c=b}
  \frac{dL(c)}{dc} \; (1-F_0(c)) \; dc
\end{align*}
The proof will simply be integration by parts, with some careful handling of the limits of integration.
\begin{proof}
\begin{align*}
  &\int_{s=0}^{s=1}
  \Bigl(
    L(\max(a,\min(b,s))) - L(\max(a,\min(b,0)))
  \Bigr)
  \; dF_0(s)
  \\&= \int_{s=0}^{s=1}
  \Bigl(
    L(\max(a,\min(b,s))) - L(a)
  \Bigr)
  \; dF_0(s)
  \\&= \int_{s=a}^{s=1}
  \Bigl(
    L(\min(b,s)) - L(a)
  \Bigr)
  \; dF_0(s)
  \\&= \int_{s=a}^{s=1}
  \Bigl(
    \int_{c=a}^{c=\min(b,s)}
    \frac{dL(c)}{dc} \; dc
  \Bigr)
  \; dF_0(s)
  \\&= \int_{c=a}^{c=b}
  \Bigl(
    \int_{s=c}^{s=1}
    \; dF_0(s)
  \Bigr)
  \frac{dL(c)}{dc} \; dc
  \\&= \int_{c=a}^{c=b}
  \frac{dL(c)}{dc}
  \Bigl(
    F_0(1) - F_0(c)
  \Bigr)
  \; dc
  \\&= \int_{c=a}^{c=b}
  \frac{dL(c)}{dc}
  \Bigl(1 - F_0(c)\Bigr)
  \; dc
\end{align*}
\end{proof}
\end{lemma}

\begin{lemma}[Combining Classes]
\label{lem:combine_classes}
\begin{align*}
  &\mathbb{E}_{x,y\sim\mathcal{D}}
  \Bigl[
    L(|y-\max(a,\min(b,s(x))|))
    - L(|y-\max(a,\min(b,y))|)
  \Bigr]
  \\&=\int_{c=a}^{c=b} \Bigl(
  \frac{dL(c)}{dc} \; (1-\pi) (1-F_0(c))
  -\frac{dL(1-c)}{dc} \; \pi \; F_1(c)
  \Bigr) \; dc
\end{align*}
\begin{proof}
The proof is a simple application of Lemma \ref{lem:pos_brier} and Lemma \ref{lem:neg_brier}.
\begin{align*}
  &\mathbb{E}_{x,y\sim\mathcal{D}}
  \Bigl[
    L(|y-\max(a,\min(b,s(x))|))
    - L(|y-\max(a,\min(b,y))|)
  \Bigr]
  \\&=
  (1-\pi)
  \int_{s=0}^{s=1}
  \Bigl[
    L(\max(a,\min(b,s(x))))
    - L(\max(a,\min(b,0)))
  \Bigr]
  \; dF_0(s)
  \\&\quad+
  \pi\;
  \int_{s=0}^{s=1}
  \Bigl[
    L(1-\max(a,\min(b,s(x))))
    - L(1-\max(a,\min(b,1)))
  \Bigr]
  \; dF_1(s)
  \\&=
  (1-\pi)
  \int_{c=a}^{s=b}
  \frac{dL(c)}{dc}
  [1-F_0(c)]
  \;dc
  -
  \pi\;
  \int_{c=a}^{s=b}
  \frac{dL(1-c)}{dc}
  F_1(c)
  \;dc
  \\&=
  \int_{c=a}^{s=b}
  \Bigl(
  \frac{dL(c)}{dc}
  (1-\pi)
  [1-F_0(c)]
  -
  \frac{dL(1-c)}{dc}
  \pi\; F_1(c)
  \Bigr)
  \;dc
\end{align*}
\end{proof}
\end{lemma}

\begin{proposition}[Propriety of Bounded-Threshold Scoring Rules]
    \label{prop:propriety}
    Let $L(s, y)$ be a strictly proper scoring rule, and let $\operatorname{clip}_{[a,b]}(s) = \max(a, \min(b, s))$.
    Then the bounded-threshold scoring rule
    \[
    L_{[a,b]}(s, y) \;=\; L\bigl(\operatorname{clip}_{[a,b]}(s),\, y\bigr) \;-\; L\bigl(\operatorname{clip}_{[a,b]}(y),\, y\bigr)
    \]
    is proper but not strictly proper.
\end{proposition}
    
\begin{proof}
  Let $S(z, q) = \mathbb{E}_{y \sim q}[L(z, y)]$ be the expected loss under true probability $q = P(y=1)$. Because $L$ is strictly proper, $S(z, q)$ is uniquely minimised at $z=q$.
  
  The expected bounded loss is $S(\operatorname{clip}_{[a,b]}(s), q)$ minus a constant independent of $s$. Assuming differentiability, the chain rule yields:
  \[
  \frac{\partial}{\partial s} S(\operatorname{clip}_{[a,b]}(s), q) 
  \;=\; \left. \frac{\partial S(z, q)}{\partial z} \right|_{z=\operatorname{clip}_{[a,b]}(s)} \cdot \mathbf{1}_{(a < s < b)}
  \]
  Inside the interval $(a,b)$, the derivative matches the strictly proper rule, vanishing if and only if $s=q$. Outside the interval, the indicator is zero, meaning the expected loss is perfectly flat. 
  
  Thus, if $q \in [a,b]$, the expected bounded loss is uniquely minimised at $s=q$. If $q < a$, the unconstrained loss strictly increases for $z > a$, meaning the bounded loss strictly increases for $s > a$ and is flat for $s < a$. Any prediction $s \le a$ (including $s=q$) achieves the identical minimal loss. A symmetric argument applies for $q > b$.
  
  Because setting $s=q$ always achieves the global minimum, the rule is proper. Because the minimiser is not unique when $q \notin [a,b]$, the rule is not strictly proper.
\end{proof}

\subsection{Specific Loss Functions}
\begin{theorem}[Bounded Threshold Brier Score]
\label{apdx:bounded_brier}
For a classifier $\f$, the integral of regret over the interval [a,b] is a the Brier Score of the predictions clipped to [a,b] minus the Brier Score of the true labels clipped to [a,b].

\begin{align*}
  \underset{c \sim \text{Uniform}(a,b)}{\mathbb{E}}\Rmin(c)
  &=
  \frac{1}{b-a}\biggl[
  \underset{(x,y)\in\mathcal{D}}{\mathbb{E}}[(y-\max(a,\min(b, s(x))))^2]
  \;-\; \underset{(x,y)\in\mathcal{D}}{\mathbb{E}}[(y-\max(a,\min(b, y)))^2]
  \biggr]
\end{align*}
\end{theorem}

\begin{proof}
Let $L(x) = x^2$ be the quadratic pointwise loss.
Then $\frac{dL(c)}{dc} = 2c$ and $-\frac{dL(1-c)}{dc} = 2(1-c)$.
\begin{align*}
  &\frac{1}{b-a}\biggl[
  \underset{(x,y)\in\mathcal{D}}{\mathbb{E}}[(y-\max(a,\min(b, s(x))))^2]
  \;-\; \underset{(x,y)\in\mathcal{D}}{\mathbb{E}}[(y-\max(a,\min(b, y)))^2]
  \biggr]
  \\&=\frac{1}{b-a}
  \underset{(x,y)\in\mathcal{D}}{\mathbb{E}}\biggl[(y-\max(a,\min(b, s(x))))^2
  \;-\; (y-\max(a,\min(b, y)))^2\biggr]
  \\&\text{Using Lemma \ref{lem:combine_classes}, we have}
  \\&=\frac{1}{b-a}\int_{c=a}^{c=b} \Bigl(
  2c\; (1-\pi)(1-F_0(c))
  + 2(1-c)\;\pi\; F_1(c)
  \Bigr) \; dc
  \\&=\frac{1}{b-a}\int_{c=a}^{c=b}
  2\Rmin(c)
  \; dc
  \\&=2 \underset{c \sim \text{Uniform}(a,b)}{\mathbb{E}}\Rmin(c)
\end{align*}
\end{proof}

\begin{theorem}[Bounded Threshold Log Loss]
\label{apdx:bounded_logloss}
For a classifier $\f$, the integral of regret over the interval [a,b] with log-odds uniform weighting is a the Log Loss of the predictions clipped to [a,b] minus the Log Loss of the true labels clipped to [a,b].

\begin{align*}
  &\underset{\ell\sim\text{Uniform}\bigl(
    \log\frac{a}{1-a},\; \log\frac{b}{1-b}
    \bigr)}{\mathbb{E}}\left[\Rmin(c=\frac{1}{1+\exp -\ell})\right]
  \\&= \frac{1}{\log\frac{b}{1-b}-\log\frac{a}{1-a}}\biggl[
    \underset{(x,y)\in\mathcal{D}}{\mathbb{E}}[\log(1-|y-\max(a,\min(b, s(x)))|)]
    \\&\hspace{10em}-\; \underset{(x,y)\in\mathcal{D}}{\mathbb{E}}[\log(1-|y-\max(a,\min(b, y))|)]
    \biggr]
\end{align*}
\end{theorem}

\begin{proof}
Let $L(x) = \log(1-x)$ be the logarithmic pointwise loss.  Then $\frac{dL(c)}{dc} = \frac{1}{1-c}$ and $-\frac{dL(1-c)}{dc} = \frac{1}{c}$.
\begin{align*}
  &\frac{1}{\log\frac{b}{1-b}-\log\frac{a}{1-a}}\biggl[
    \underset{(x,y)\in\mathcal{D}}{\mathbb{E}}[\log(1-|y-\max(a,\min(b, s(x)))|)]
    \\&\hspace{10em}-\; \underset{(x,y)\in\mathcal{D}}{\mathbb{E}}[\log(1-|y-\max(a,\min(b, y))|)]
    \biggr]
  \\&= \frac{1}{\log\frac{b}{1-b}-\log\frac{a}{1-a}}
    \underset{(x,y)\in\mathcal{D}}{\mathbb{E}}\biggl[
      \log(1-|y-\max(a,\min(b, s(x)))|)
      \;-\; \log(1-|y-\max(a,\min(b, y))|)
    \biggr]
  \\&\text{Using Lemma \ref{lem:combine_classes}, we have}
  \\&= \frac{1}{\log\frac{b}{1-b}-\log\frac{a}{1-a}}
  \int_{c=a}^{c=b} \Bigl(
  \frac{1}{1-c}\; (1-\pi)(1-F_0(c))
  + \frac{1}{c}\;\pi\; F_1(c)
  \Bigr) \; dc
  \\&= \frac{1}{\log\frac{b}{1-b}-\log\frac{a}{1-a}}
  \int_{c=a}^{c=b} \Bigl(
  c\; (1-\pi)(1-F_0(c))
  + (1-c)\;\pi\; F_1(c)
  \Bigr) \; \frac{dc}{c(1-c)}
  \\&= \frac{1}{\log\frac{b}{1-b}-\log\frac{a}{1-a}}
  \int_{c=a}^{c=b} \Rmin(c)
  \; \frac{dc}{c(1-c)}
\end{align*}
Now we do a change of variables $\ell = \log\frac{c}{1-c}$, $\frac{d\ell}{dc} = \frac{1}{c(1-c)}$.
\begin{align*}
  &\frac{1}{\log\frac{b}{1-b}-\log\frac{a}{1-a}}
  \int_{c=a}^{c=b} \Rmin(c)
  \; \frac{dc}{c(1-c)}
  \\&= \frac{1}{\log\frac{b}{1-b}-\log\frac{a}{1-a}}
  \int_{\ell=\log\frac{a}{1-a}}^{\ell=\log\frac{b}{1-b}} \Rmin(c=\frac{1}{1+\exp -\ell})
  \; d\ell
  \\&= \underset{\ell\sim\text{Uniform}\bigl(
    \log\frac{a}{1-a},\; \log\frac{b}{1-b}
    \bigr)}{\mathbb{E}}\left[\Rmin(c=\frac{1}{1+\exp-\ell})\right]
\end{align*}
\end{proof}

\subsection{Shifted Brier Score}

\begin{definition}[Score Adjustment]
Let \( s \in (0,1) \) be a predicted probability and let \( \mu \in (0,1) \) denote a reference class probability.
Define the \emph{score adjustment function} \( M: (0,1) \times (0,1) \to (0,1) \) as:
\[
M(s, \mu) \triangleq \frac{1}{1 + \exp\left( \log\left(\frac{s}{1 - s}\right) - \log\left(\frac{\mu}{1 - \mu}\right) \right)}
\]
That is, \( M(s, \mu) \) adjusts the predicted log-odds of \( s \) by centering it around the log-odds of \( \mu \).

We extend \( M \) to the boundary values \( s \in \{0,1\} \) by defining:
\[
\lim_{s \to a} M(s, \mu) = a \quad \text{for } a \in \{0,1\}
\]
\end{definition}

\begin{proposition}[Inverse of Score Adjustment]
$M(M(s,\mu),-\mu) = s$
\end{proposition}

\begin{lemma}
\label{lem:score_adjustment}
Let $G(s,y) : [0,1] \times \{0,1\} \to [0,1]$ be the cumulative distribution function of $s$ for either the positive or negative class, and let $G(0,y) = 0$ and $G(1,y) = 1$.
\begin{align*}
  &\int_{s=0}^{s=1} (y-M(s, -\mu))^2 \; dG(s,y)
  \\&= \int_{s=0}^{s=1} \int_{c=M(y,-\mu)}^{c=M(s,-\mu)} -2(y-c) \; dc \; dG(s,y)
  \\&= -2\int_{c=0}^{c=1} (y-c)\; \int_{s=M(c,+\mu)}^{s=M(1-y,+\mu)}  dG(s,y) \; dc
  \\&= -2\int_{c=0}^{c=1} (y-c)\; [G(M(1-y, +\mu),y) - G(M(c, +\mu),y)] dc \;
  \\&= -2\int_{c=0}^{c=1} (y-c)\; [1-y - G(M(c, +\mu),y)] dc \;
\end{align*}
\end{lemma}

\begin{theorem}
\label{apdx:shifted_brier}
If we define a new score such that $s'(x) = \text{M}(s(x),\mu)$, then
\[
\underset{(x,y)\in\mathcal{D}}{\mathbb{E}}(y-s'(x))^2
=
\underset{c \sim \text{Uniform}(0,1)}{\mathbb{E}}\Rmin(M(c,\mu))
\]
\begin{proof}
Let $G(s,y) = \begin{cases}
  F_0(s) & \text{if } y = 0
  \\F_1(s) & \text{if } y = 1
\end{cases}$.
Then using Lemma \ref{lem:score_adjustment}, twice we have:
\begin{align*}
  &\int_{s=0}^{s=1} (0-M(s(x),+\mu))^2 \; dF_0(s) + \int_{s=0}^{s=1} (1-M(s(x),+\mu))^2 \; dF_1(s)
  \\&=2\int_{c=0}^{c=1} (0-c) [1 - 0 - F_0(M(c, -\mu))] + (1-c) [1 - 1 - F_1(M(c, -\mu))] \; dc
  \\&=-2\int_{c=0}^{c=1} c [1 - F_0(M(c, -\mu))] + (1-c) [F_1(M(c, -\mu))] \; dc
  \\&=-2\int_{c=0}^{c=1} \Rmin(M(c, -\mu)) \; dc
\end{align*}
\end{proof}
\end{theorem}

%% file: apdx/dca.tex
%
\section{Average Net Benefit as a Proper Scoring Rule}
\label{apdx:checklist_dca}

\begin{theorem}[Restatement of Theorem \ref{thm:net_benefit_h}]
\label{apdx:dca}
\[
  \text{\em NB}(c) = \pi - \frac{\Rmin(c)}{1-c}
\]
\begin{proof}
Once we express the net benefit definition given in \cite{vickers19dca} using the terminology of this paper and arrange the terms, the result follows.
\begin{align*}
  \text{\em NB}(c) &= \text{sensitivity} \times \text{prevalence} - (1 - \text{specificity}) \times (1 - \text{prevalence}) \times \frac{\TAU}{1-\TAU}
  \\&= (1-F_1(\TAU))  \pi - (1-F_0(\TAU))  (1 - \pi)  \frac{\TAU}{1-\TAU}
  \\&= \frac{1}{1-\TAU}\biggl[
    (1-\TAU)(1-F_1(\TAU))  \pi - (1-F_0(\TAU))  (1 - \pi)  \TAU
  \biggr]
  \\&= \frac{1}{1-c}\biggl[
    (1-c)(1-F_1(c))  \pi - (1-F_0(c))  (1 - \pi)  c
  \biggr]
  \\&= \pi - \frac{1}{1-c}\biggl[
    (1-c) F_1(c)  \pi + (1-F_0(c))  (1 - \pi)  c
  \biggr]
  \\&= \pi - \frac{\Rmin(c)}{1-c}
\end{align*}
\end{proof}
\end{theorem}

\begin{theorem}[Restatement of Theorem \ref{thm:bounded_dca}]
\label{apdx:bounded_dca}
Let $L(x,y) = \begin{cases}
  s(x) & \text{if } y=1 \\
  (1 - s(x)) - \ln(1-s(x)) & \text{if } y=0
\end{cases}$ be a pointwise loss.

For a classifier $\f$, the integral of net benefit over the interval [a,b] is the loss for the predictions clipped to [a,b] minus the loss for the true labels clipped to [a,b].

\begin{align*}
  \underset{c \sim \text{Uniform}(a,b)}{\mathbb{E}}\text{\em NB}(c)
  &=
  \pi -
  \frac{1}{b-a}\biggl[
  \underset{(x,y)\in\mathcal{D}}{\mathbb{E}}
  L(max(a,\min(b, s(x))),y)
  \;-\; \underset{(x,y)\in\mathcal{D}}{\mathbb{E}} L(\max(a,\min(b, y)),y)
  \biggr]
\end{align*}
\end{theorem}

\begin{proof}

Note that
$\frac{dL(x,y)}{dc} = \begin{cases}
  1 & \text{if } y=1 \\
  \frac{c}{1-c} & \text{if } y=0
\end{cases}$.  Then.
\begin{align*}
  &\pi - \frac{1}{b-a}\biggl[
  \underset{(x,y)\in\mathcal{D}}{\mathbb{E}}[L(\max(a,\min(b, s(x), y)))]
  \;-\; \underset{(x,y)\in\mathcal{D}}{\mathbb{E}}[L(\max(a,\min(b, y)),y)]
  \biggr]
  \\&\text{Using Lemma \ref{lem:combine_classes}, we have}
  \\&=\pi - \frac{1}{b-a}\int_{c=a}^{c=b} \Bigl(
  \frac{dL(c,0)}{dc}\; (1-\pi)(1-F_0(c))
  + \frac{dL(c,1)}{dc}\;\pi\; F_1(c)
  \Bigr) \; dc
  \\&=\pi - \frac{1}{b-a}\int_{c=a}^{c=b}
  \frac{\Rmin(c)}{1-c}
  \; dc
  \\&=\underset{c \sim \text{Uniform}(a,b)}{\mathbb{E}}\pi - \frac{\Rmin(c)}{1-c}
  \\&=\underset{c \sim \text{Uniform}(a,b)}{\mathbb{E}}\text{ NB}(c)
\end{align*}
\end{proof}

\begin{figure}[ht]
\begin{center}
\includegraphics[width=0.8\columnwidth]{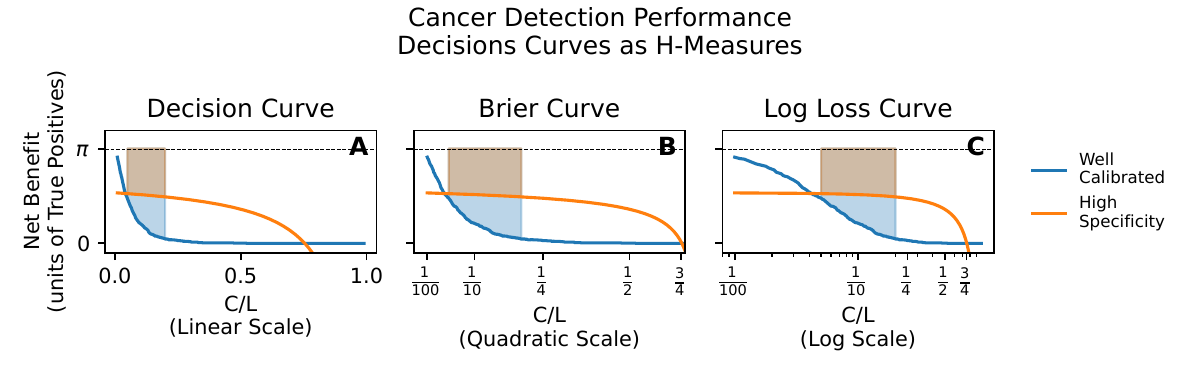}
\vskip -0.1in
\caption{
The figure shows the DCA (A), which can be rescaled so that for an interval of cost ratios, the area above the curve and below the prevalence $\pi$ is equal to the bounded threshold Brier score (B) or bounded threshold log loss (C).
}
\label{fig:dca}
\end{center}
\end{figure}

\begin{theorem}[Quadratically Rescaled Decision Curve]
\label{apdx:quad_dca}
Let $\phi(c) \triangleq \frac{-(1-c)^2}{2}$ and therefore $\frac{d\phi(c)}{dc} = 1-c$.
Note that this is invertible on the interval $[0,1]$.
\begin{align*}
  &\frac{1}{b-a}\int_{x=\phi(a)}^{x=\phi(b)} \pi - \text{NB}(x) \; dx
  \\&\text{Using Theorem \ref{apdx:bounded_dca}, we have}
  \\&=\frac{1}{b-a}\int_{x=\phi(a)}^{x=\phi(b)} \frac{\Rmin(\phi^{-1}(x))}{1-\phi^{-1}(x)} \; dx
  \\&=\frac{1}{b-a}\int_{c=a}^{c=b} \frac{\Rmin(c)}{1-c} \; (1-c)dc
  \\&=\underset{c \sim \text{Uniform}(a,b)}{\mathbb{E}}\Rmin(c)
\end{align*}
\end{theorem}

\begin{theorem}[Logarithmically Rescaled Decision Curve]
\label{apdx:log_dca}
Let $\phi(c) \triangleq \ln c$ and therefore $\frac{d\phi(c)}{dc} = \frac{1}{c}$.
Note that this is invertible on the interval $(0,1]$.
\begin{align*}
  &\frac{1}{\log\frac{b}{1-b}-\log\frac{a}{1-a}} \int_{x=\phi(a)}^{x=\phi(b)} \pi - \text{NB}(x) \; dx
  \\&\text{Using Theorem \ref{apdx:bounded_dca}, we have}
  \\&=\frac{1}{\log\frac{b}{1-b}-\log\frac{a}{1-a}} \int_{x=\phi(a)}^{x=\phi(b)} \frac{\Rmin(\phi^{-1}(x))}{1-\phi^{-1}(x)} \; dx
  \\&=\frac{1}{\log\frac{b}{1-b}-\log\frac{a}{1-a}} \int_{c=a}^{c=b} \frac{\Rmin(c)}{1-c} \; \frac{dc}{c}
  \\&= \underset{\ell\sim\text{Uniform}\bigl(
    \log\frac{a}{1-a},\; \log\frac{b}{1-b}
    \bigr)}{\mathbb{E}}\left[\Rmin(c=\frac{1}{1+\exp-\ell})\right]
\end{align*}
\end{theorem}

\section{Comparison of Results from \texorpdfstring{\citet{vickers17}}{Vickers et al. (2017)}}
\label{apdx:checklist_vickers}

\citet{vickers17} provides a table of results showing that the ordering of model quality according to Brier Score fails to match the ordering of model quality according to Net Benefit at a 5\% threshold.
We reproduce their data generating process and show the same table, with results sorted by the Net Benefit at a 5\% threshold for convenience.
Note that the ordering by overall Brier Score is indeed quite different.
But the ordering by Bounded Threshold Brier Score is almost the same.
The single, instructive exception is Assume All Positive, where the Net Benefit at a 20\% threshold sharply disagrees with Net Benefit at a 5\% threshold.
In this case, the Bounded Threshold Brier Score puts some weight on these higher threshold cases.

\label{apdx:comparison_vickers17}
\begin{tabular}{lllllllll}
  \toprule
   & AUC-ROC & Brier & NB & NB & NB & Brier \\
  test & & & 5\% & 10\% & 20\% & 5\%-20\% \\
  \midrule
  \textbf{Highly sensitive} & 0.73 & 0.41 & 0.17 & 0.15 & 0.09 & 0.12 \\
  \textbf{Underestimating risk} & 0.75 & 0.15 & 0.16 & 0.12 & 0.06 & 0.16 \\
  \textbf{Well calibrated} & 0.75 & 0.17 & 0.16 & 0.12 & 0.05 & 0.17 \\
  \textbf{Overestimating risk} & 0.75 & 0.20 & 0.16 & 0.11 & 0.03 & 0.18 \\
  \textbf{Assume all positive} & 0.50 & 0.80 & 0.16 & 0.11 & 0.00 & 0.20 \\
  \textbf{Highly specific} & 0.73 & 0.14 & 0.10 & 0.10 & 0.09 & 0.18 \\
  \textbf{Severely underestimating risk} & 0.75 & 0.18 & 0.09 & 0.04 & 0.01 & 0.29 \\
  \textbf{Assume all negative} & 0.50 & 0.20 & 0.00 & 0.00 & 0.00 & 0.35 \\
  \bottomrule
\end{tabular}

Because the data is synthetic, we were able to generate this table based on 1,000,000 samples.
The largest standard deviation for any cell in this table is 0.00035, two decimal places further than the precision we have reported for the cells, and as such, standard errors have been omitted from the table cells for clarity.

%% file: apdx/costevaluation.tex
\section{Choices of Threshold Distributions}
\label{apdx:threshold_distributions}

\begin{figure}[h!]
\begin{center}
\includegraphics[width=.9\columnwidth]{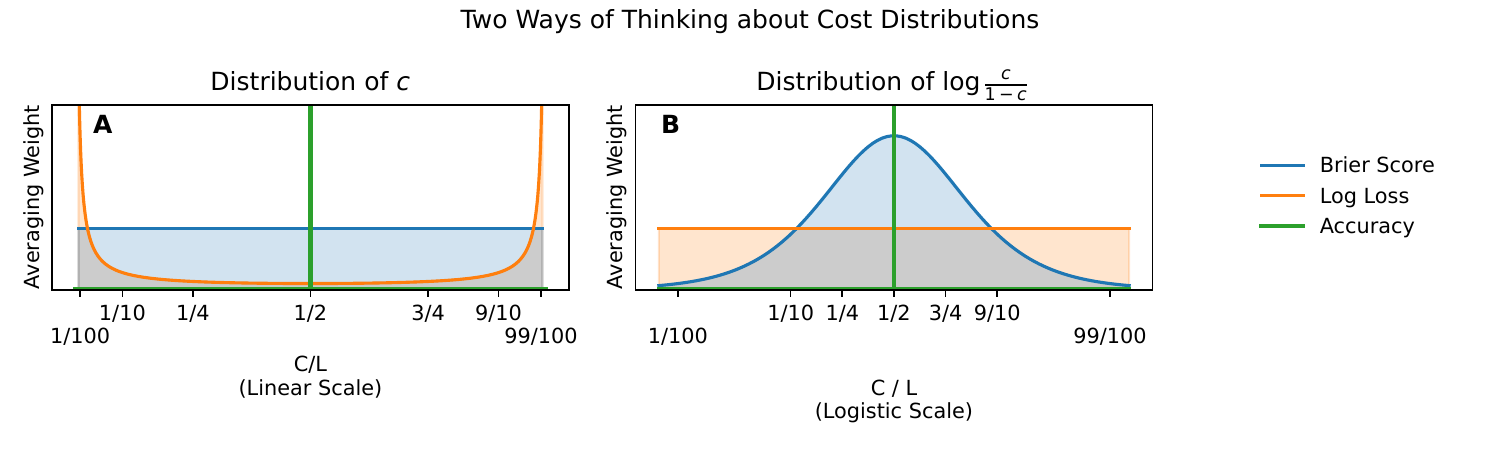}
\vskip -0.1in
\caption{
\small Where the tradeoff between false positives and false negatives is one-to-one, accuracy is a good metric.
Where trade-offs are generally moderate, Brier score is a good metric.
Where trade-offs are frequently extreme, log loss is better.
}
\label{fig:prior_weights}
\end{center}
\end{figure}

\begin{figure}[ht]
\centering
\includegraphics[width=0.95\columnwidth]{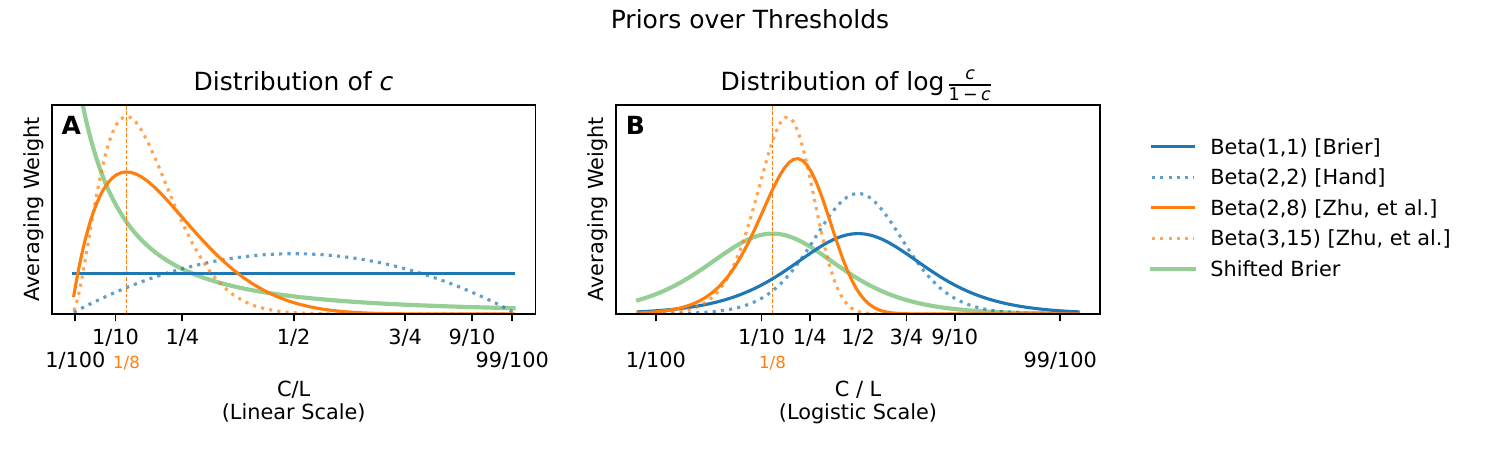}
\vspace{-0.1in}
\caption{\small
Comparison of cost ratio priors implicit in mixture-of-thresholds metrics. Brier score assumes \(\text{Beta}(1,1)\); Hand proposes increasing concentration with \(\text{Beta}(2,2)\); \citet{recentbeta24} shifts the mode while inheriting concentration challenges.
}
\label{fig:hmeasure}
\end{figure}

Interest in cost-sensitive evaluation during the late 1990s brought renewed attention to the Brier score.
\citet{hand99} noted that while domain experts rarely specify exact cost ratios, they can often provide plausible bounds.
To improve interpretability, he proposed the LC-Index, which ranks models at each cost ratio and plotting their ranks across the range.
Later, \citet{hand09} introduced the more general \textbf{H-measure}, defined as any weighted average of regret, and recommended a \(\text{Beta}(2,2)\) prior to emphasize cost ratios near \(c = 0.5\).

Despite its appeal, the H-measure's intuition can be opaque: even the \(\text{Beta}(1,1)\) prior used by the Brier score already concentrates mass near parity on the log-odds scale (Figure~\ref{fig:hmeasure}).

\citet{recentbeta24} generalize this idea using asymmetric Beta distributions centered at an expert-specified mode (e.g., \(\text{Beta}(2,8)\)).
However, this raises concerns: the mode is not invariant under log-odds transformation, may be less appropriate than the mean, and requires domain experts to specify dispersion—a difficult task in practice.
A simpler alternative is to shift the Brier score to peak at the desired cost ratio via a transformation of the score function \(s(x)\), as shown in Appendix~\ref{apdx:shifted_brier}.

Rather than infer uncertainty via a prior, \citet{recentbeta24} suggest eliciting threshold bounds directly (e.g., from clinicians).
We argue that this approach is better served by constructing explicit threshold intervals rather than encoding beliefs via Beta distributions.

%% file: apdx/pr.tex
%
\section{Dependent Decisions with Fixed \texorpdfstring{$K$}{K}: Equivalence of Top-\texorpdfstring{$K$}{K} Metrics}
\label{apdx:fixed_k}

When the capacity $K$ is known exactly at model selection time, a metric like Net Benefit@$K$ or True Positives@$K$ (TP@$K$) is preferable to AUC-ROC. Although Net Benefit@$K$ is conceptually superior because it grounds the evaluation in actual costs, it induces the exact same model ranking as standard ranking metrics like Precision@$K$ and Recall@$K$ for a fixed dataset, fixed budget $K$, and fixed cost ratio $c$.

To see this algebraically, let $P$ be the total number of actual positives in the evaluation dataset. If a model is constrained to predict exactly $K$ positives, let $TP$ denote the number of true positives it selects. The number of false positives is constrained to be $FP = K - TP$, and the number of false negatives left in the unselected pool is $FN = P - TP$.

Under our expected regret formulation, the total cost incurred by the decision rule is:
\begin{align*}
\text{Cost} &= c \cdot FP + (1-c) \cdot FN \\
&= c(K - TP) + (1-c)(P - TP) \\
&= cK - cTP + P - TP - cP + cTP \\
&= cK + (1-c)P - TP
\end{align*}

Because the cost ratio $c$, the budget $K$, and the total positives $P$ are all fixed constants for a given evaluation, the total cost collapses to a constant minus $TP$. 

Consequently, minimizing cost (or maximizing Net Benefit) under a fixed budget is mathematically identical to maximizing $TP$. Because Precision@$K = TP / K$ and Recall@$K = TP / P$ are merely positive scalar multiples of $TP$, all of these metrics are affine transformations of one another and will induce the exact same model ranking.

\section{Bounded Threshold AUC-ROC}
\label{apdx:auc}

A natural question to ask is whether the idea of bounding thresholds can be applied to make the AUC-ROC more useful, the same way it makes the Brier score more useful.

To investigate this question, we review the existing attempts to bound the AUC-ROC, and show a new and conceptually clearer way to do it.  We then evaluate the strengths and weakness of this approach, and conclude that it is better to use the bounded threshold Brier score instead.  For this reason the development of this new bounded threshold AUC-ROC is left to the Appendix rather than the main text.

\subsection{Outline}
\citet{hand09} showed that, under the strong assumption of calibration, the AUC-ROC is equivalent to an integral of expected regret where for our distribution of users' costs, we use the distribution of scores from the model.

We extend this result to show that we can limit the cost distribution over which we take an expectation to those scores from the model that lie in a given interval.

We will first show that we can convert the averages $pAUC_+$ and $pAUC_-$ into expectations over a single distribution.  We then combine the terms, and find that they give expected regret plus a couple of simpler terms.

Under calibration, the Bayes-optimal threshold equals the cost ratio $c$
(via $R_{\min}(\tau)=(1-\pi)\tau(1-F_0)+\pi(1-\tau)F_1$ and the first-order
condition), so bounding the threshold interval $[a,b]$ is equivalent to
bounding plausible cost ratios $c\in[a,b]$ in deployment.

\subsection{Background}

Our work builds on the ideas of pAUC from \cite{mcclish89} which average TPR based on bounding the FPR, but lacks decision-theoretic interpretation.  It is not clear why we would want to limit a range of FPR, and it is not clear why a uniform distribution of FPR produces a meaningful average of TPR.

It also builds on the ideas of PAI from \cite{jiang96} which average FPR based on bounding the TPR, but has many of the same problems.

Finally, our work also builds on the idea of Balanced Average Accuracy or Concordant Partial AUC from \cite{balancedaccuracy23} which combines those two averages assuming balanced classes, and adds the idea of bounding by a quantile of the score rather than FPR or TPR.

We add the idea of bounding by the costs encountered in deployment, rather than a quantile of scores, which is a better model of the deployment scenario.  Our measure also has a direct interpretation as expected regret, averaged over a single distribution rather than two.

\subsection{Notation}

Let $F_0$ and $F_1$ be the score CDFs for the negative and positive classes with
densities $f_0=\frac{d}{dc}F_0$ and $f_1=\frac{d}{dc}F_1$. Let
\[
F(c) \;=\; (1-\pi)F_0(c) + \pi F_1(c), \qquad
f(c) \;=\; \frac{d}{dc}F(c) \;=\; (1-\pi)f_0(c) + \pi f_1(c),
\]
be the mixture CDF and density.

Note that the True Positive Rate is the fraction of positive cases where the score exceeds the threshold, so $TPR = 1 - F_1(\tau)$, which can be counterintuitive.  Likewise, False Positive Rate is the fraction of negative cases where the score is greater than the threshold, so $FPR = 1-F_0(\tau)$.

We also define two quantities for the area beneath a subset of the ROC curve as well as the area to the right.

\begin{definition}[$pAUC_+$]
This represents the area under a portion of the ROC curve, bounded by the thresholds $a$ and $b$.
$$pAUC_+(a,b)
= \int_{\tau = a}^{b} (1-F_1(\tau)) dF_0(\tau)
$$
\end{definition}

Note that this is closely related to the partial AUC metric from \cite{mcclish89}, although FPR is decreasing in the threshold $\tau$, so the order of the bounds is reversed.  Also our function is defined in terms of the threshold value, rather than the FPR value.
$$ \text{Partial AUC}(\alpha, \beta) = \int_{FPR = \alpha}^{FPR = \beta}
TPR(FPR)\; dFPR = pAUC_+(FPR^{-1}(\beta), FPR^{-1}(\alpha))$$

\begin{definition}[$pAUC_-$]
This represents the area to the right of a portion of the ROC curve, bounded by the thresholds $a$ and $b$.
$$pAUC_-(a,b)
= \int_{\tau=a}^{b} F_0(\tau)dF_1(\tau)
$$
\end{definition}

Note that this is closely related to the partial area index metric from \cite{jiang96}, although FPR is decreasing in the threshold $\tau$, so the order of the bounds is reversed.  Also our function is defined in terms of the threshold value, rather than the TPR value.

$$ \text{PAI}(\alpha, \beta) = \int_{TPR = \alpha}^{TPR = \beta}
(1-FPR(TPR))\; dTPR = pAUC_-(TPR^{-1}(\beta), TPR^{-1}(\alpha))
$$

\subsection{Main Result}

We first convert the averages $pAUC_+$ and $pAUC_-$ taken over positive or negative class scores into averages over all scores.

Under calibration, the local shares satisfy
$\frac{\pi dF_1(\tau)}{dF(\tau)}=\tau$
and $\frac{(1-\pi)dF_0(\tau)}{dF(\tau)}=1-\tau$.

\begin{lemma}
\label{lemma:pAUCplus}
$pAUC_+$ is equivalent to a threshold-weighted average of $1-F_1(\tau)$ over the scores from the combined distribution of positive and negative examples, up to a constant factor.
\begin{align*}
  \frac{\pi(1-\pi)}{F(b)-F(a)}\;
  pAUC_+(a,b)
  = \underset{x,y \sim \mathcal{D} | s(x) \in [a,b]}{\mathbb{E}}[\pi\;(1-s(x))\;(1-F_1(s(x)))]
\end{align*}
\begin{proof}
We use the assumption of calibration to do a change of variables.
\begin{align*}
  pAUC_+(a,b)
  &= \int_{\tau = a}^{b} (1-F_1(\tau)) dF_0(\tau)
  \\&= \frac1{1-\pi}\int_{\tau = a}^{b} (1-F_1(\tau)) \frac{(1-\pi) dF_0(\tau)}{dF(\tau)} dF(\tau)
  \\&= \frac1{1-\pi}\int_{\tau = a}^{b} (1-F_1(\tau)) (1-\tau) dF(\tau)
  \\&= \frac{F(b)-F(a)}{1-\pi}\;\underset{\tau \sim F | \tau \in [a,b]}{\mathbb{E}}[(1-F_1(\tau))\;(1-\tau)]
  \\&= \frac{F(b)-F(a)}{\pi(1-\pi)}\;\underset{x,y \sim \mathcal{D} | s(x) \in [a,b]}{\mathbb{E}}[\pi\;(1-s(x))\;(1-F_1(s(x)))]
\end{align*}
\end{proof}
\end{lemma}

An equivalent identity holds for the negative class.

\begin{lemma}
  \label{lemma:pAUCminus}
  $pAUC_-$ is equivalent to a threshold-weighted average of $F_0(\tau)$ over the scores from the combined distribution of positive and negative examples, up to a constant factor.
  \begin{align*}
    \frac{\pi(1-\pi)}{F(b)-F(a)}\;
    pAUC_-(a,b)
    = \underset{x,y \sim \mathcal{D} | s(x) \in [a,b]}{\mathbb{E}}[(1-\pi)\;s(x)\;F_0(s(x))]
  \end{align*}
  \begin{proof}
We use the assumption of calibration to do a change of variables.
  \begin{align*}
    pAUC_-(a,b)
    &= \int_{\tau = a}^{b} F_0(\tau) dF_1(\tau)
    \\&= \frac{1}{\pi}\int_{\tau = a}^{b} F_0(\tau) \frac{\pi dF_1(\tau)}{dF(\tau)} dF(\tau)
    \\&= \frac{1}{\pi}\int_{\tau = a}^{b} F_0(\tau)\; \tau\; dF(\tau)
    \\&= \frac{F(b)-F(a)}{\pi}\;\underset{\tau \sim F | \tau \in [a,b]}{\mathbb{E}}[F_0(\tau)\;\tau]
    \\&= \frac{F(b)-F(a)}{\pi(1-\pi)}\;\underset{x,y \sim \mathcal{D} | s(x) \in [a,b]}{\mathbb{E}}[(1-\pi)\;s(x)\;F_0(s(x))]
  \end{align*}
\end{proof}
\end{lemma}

Since they are now expectations over the same measure, we can combine them.

\begin{theorem}
\label{thm:bounded_auc}
The weighted sum of $pAUC_+$ and $pAUC_-$ is equivalent to negative expected regret, plus a couple of simpler additive terms.

$$
    \frac{\pi(1-\pi)}{F(b)-F(a)}\;
    \left[pAUC_+(a,b) + pAUC_-(a,b)\right]
    = \underset{x,y \sim \mathcal{D} | s(x) \in [a,b]}{\mathbb{E}}[-\Rmin(s(x))] + \pi + (1-2\pi)\underset{x,y \sim \mathcal{D} | s(x) \in [a,b]}{\mathbb{E}}[y]
$$
\begin{proof}
Recall $\Rmin(\tau) = \pi\;(1-s(x))\;F_1(s(x)) + (1-\pi)\;s(x)\;(1-F_0(s(x)))$.

We substitute according to Lemmas \ref{lemma:pAUCplus} and \ref{lemma:pAUCminus}.
\begin{align*}
&\frac{\pi(1-\pi)}{F(b)-F(a)}\;
\left[pAUC_+(a,b) + pAUC_-(a,b)\right]
\\&= \underset{x,y \sim \mathcal{D} | s(x) \in [a,b]}{\mathbb{E}}[\overbrace{\pi\;(1-s(x))\;(1-F_1(s(x)))}^{\text{from } pAUC_+} + \overbrace{(1-\pi)\;s(x)\;F_0(s(x))}^{\text{from } pAUC_-}]
\\&= \mathbb{E}[\overbrace{\pi(1-s(x))-\pi\;(1-s(x))\;F_1(s(x))}^{\text{from } pAUC_+} + \overbrace{(1-\pi)s(x) - (1-\pi)\;s(x)\;(1-F_0(s(x)))}^{\text{from } pAUC_-}]
\\&= \mathbb{E}[\overbrace{-\pi\;(1-s(x))\;F_1(s(x))}^{\text{from } pAUC_+} - \overbrace{(1-\pi)\;s(x)\;(1-F_0(s(x)))}^{\text{from } pAUC_-} + \pi(1-s(x))+(1-\pi)s(x)]
\\&= \mathbb{E}[-\Rmin(s(x)) + \pi + s(x) - 2\pi s(x)]
\end{align*}
We use the assumption of calibration to simplify the last term.
$\mathbb{E}[y|s(x)] = s(x) \implies \mathbb{E}[s(x)] = \mathbb{E}[y]$
\begin{align*}
&\frac{\pi(1-\pi)}{F(b)-F(a)}\;
\left[pAUC_+(a,b) + pAUC_-(a,b)\right]
\\&= \mathbb{E}[-\Rmin(s(x)) + \pi + s(x) - 2\pi s(x)]
\\&= \mathbb{E}[-\Rmin(s(x)) + \pi + (1-2\pi)s(x)]
\\&= \mathbb{E}[-\Rmin(s(x))] + \pi + (1-2\pi)\mathbb{E}[y]
\end{align*}
\end{proof}
\end{theorem}

\subsection{Generalization of Existing Metrics}

\begin{theorem}[Hand's Theorem]
\label{thm:hand}
If we don't clip at all (set $(a,b) = (0,1)$), then we get exactly Hand's theorem.

$$AUC = 1 - \frac{1}{2\pi(1-\pi)}\mathbb{E}[\Rmin(c)]$$

\begin{proof}

\begin{align*}
pAUC_-(a,b) + pAUC_+(a,b)
&= \frac{F(b)-F(a)}{\pi(1-\pi)}\left[
  \underset{x,y \sim \mathcal{D} | s(x) \in [a,b]}{\mathbb{E}}[-\Rmin(s(x))] + \pi + (1-2\pi)\underset{x,y \sim \mathcal{D} | s(x) \in [a,b]}{\mathbb{E}}[y]
  \right]
\\&= \frac{1}{\pi(1-\pi)}
\left[
\underset{x,y \sim \mathcal{D} | s(x) \in [a,b]}{\mathbb{E}}[-\Rmin(s(x))] + \pi + (1-2\pi)\pi
\right]
\\&= \frac{1}{\pi(1-\pi)}\left[
  \underset{x,y \sim \mathcal{D} | s(x) \in [a,b]}{\mathbb{E}}[-\Rmin(s(x))] + 2\pi(1-\pi)
  \right]
\\AUC + AUC &= - \frac{1}{\pi(1-\pi)}\mathbb{E}[\Rmin(c)] + 2
\\AUC &= 1 - \frac{1}{2\pi(1-\pi)}\mathbb{E}[\Rmin(c)]
\end{align*}
\end{proof}
\end{theorem}

\begin{definition}
Define $\Delta FPR = FPR(a) - FPR(b) = F_0(b) - F_0(a)$.
\end{definition}

\begin{definition}
Define $\Delta TPR = TPR(b) - TPR(a) = F_1(b) - F_1(a)$.
\end{definition}

\begin{definition}
The Balanced Average Accuracy of Concordant Partial AUC is defined in \citet{balancedaccuracy23} as
$$
BAA(a,b) =
\overbrace{
\frac{\Delta FPR}{\Delta TPR + \Delta FPR}
}^{\text{Fraction of Negative Cases}}
\quad
\overbrace{
\frac{pAUC_+(a,b)}{\Delta FPR}
}^{\text{Average Positive Accuracy}}
+
\overbrace{
\frac{\Delta TPR}{\Delta TPR + \Delta FPR}
}^{\text{Fraction of Positive Cases}}
\quad
\overbrace{\frac{pAUC_-(a,b)}{\Delta TPR}}^{\text{Average Negative Accuracy}}
$$
\end{definition}

\begin{lemma}
We can simplify this and express it using our notation
\begin{align*}
BAA(a,b)
&= \frac1{\Delta TPR + \Delta FPR}\biggl[{pAUC_+(a,b) + pAUC_-(a,b)}\biggr]
\\&= \frac1{F_1(b) - F_1(a) + F_0(b) - F_0(a)}\biggl[{pAUC_+(a,b) + pAUC_-(a,b)}\biggr]
\end{align*}
\end{lemma}

\begin{theorem}[Concordant Partial AUC / Balanced Average Accuracy]

When $\pi=\frac12$, then the concordant partial AUC / balanced average accuracy is a linear transformation of expected regret.

$$
BAA(a,b)
= 1 - 2\,\underset{x,y \sim \mathcal{D} | s(x) \in [a,b]}{\mathbb{E}}[\Rmin(s(x))]
$$

\begin{proof}
We first expand the definition using Theorem \ref{thm:bounded_auc}.
\begin{align*}
BAA
&= \frac1{F_1(b) - F_1(a) + F_0(b) - F_0(a)}\biggl[{pAUC_+(a,b) + pAUC_-(a,b)}\biggr]
\\&=
    \frac1{F_1(b) - F_1(a) + F_0(b) - F_0(a)}
    \frac{F(b)-F(a)}{\pi(1-\pi)}\;
    \underset{x,y \sim \mathcal{D} | s(x) \in [a,b]}{\mathbb{E}}[-\Rmin(s(x))] + \pi + (1-2\pi)\underset{x,y \sim \mathcal{D} | s(x) \in [a,b]}{\mathbb{E}}[y]
\\&=
    \frac{F(b)-F(a)}{F_1(b) - F_1(a) + F_0(b) - F_0(a)}
    \frac1{\pi(1-\pi)}\;
    \underset{x,y \sim \mathcal{D} | s(x) \in [a,b]}{\mathbb{E}}[-\Rmin(s(x))] + \pi + (1-2\pi)\underset{x,y \sim \mathcal{D} | s(x) \in [a,b]}{\mathbb{E}}[y]
\end{align*}

Now consider what happens when $\pi = \frac12$.

$$\frac{F(b)-F(a)}{F_1(b) - F_1(a) + F_0(b) - F_0(a)}
= \frac{\pi(F_1(b) - F_1(a)) + (1-\pi)(F_0(b) - F_0(a))}{F_1(b) - F_1(a) + F_0(b) - F_0(a)}
= \frac12$$

$$\frac1{\pi(1-\pi)} = 4$$

$$\pi + (1-2\pi)\mathbb{E}[y] = \frac12 + 0\cdot\mathbb{E}[y] = \frac12$$

Substituting these values back in, we get:
\begin{align*}
  BAA(a,b) &= \frac12\cdot4\cdot\left[\mathbb{E}[-\Rmin(s(x))] + \frac12\right]
  \\&= 1- 2\mathbb{E}[\Rmin(s(x))]
\end{align*}
\end{proof}
\end{theorem}

\subsection{Discussion}

Thus in the case that we know the classifier is calibrated, we can use a combination of the area beneath and to the right of a portion of the AUC curve to get an expected regret.  This will still be taking an expectation over the scores of data points whose scores lie in the indicated interval.  If there are no such points, then this will not work.

However, this representation raises a conceptual concern: it uses predicted probabilities, intended to estimate outcome likelihoods, as implicit estimates of cost ratios.
As \citet{hand09} observes, this allows the model to determine the relative importance of false positives and false negatives,
\emph{implicitly allowing the model to determine how costly it is to miss a cancer diagnosis, or how acceptable it is to let a guilty person go free.}
The model, however, is trained not to encode values or ethical trade-offs, but to estimate outcomes.
Using its scores to induce a cost distribution embeds assumptions about harms and preferences that it was never intended to model.
While a calibrated model ensures that the mean predicted score equals the class prevalence \(\pi\), there is no principled reason to treat \(\pi\) as an estimate of the true cost ratio \(c\).
Rare outcomes are not necessarily less costly, and often the opposite is true.

This analysis underscores the broader risk of deferring normative judgments, about cost, harm, and acceptability, to statistical models.
A more appropriate approach would involve eliciting plausible bounds on cost ratios from domain experts during deployment, rather than allowing the score distribution of a trained model to implicitly dictate them.
Finally, this equivalence assumes calibration, which is frequently violated in practice.
Metrics that rely on this assumption may be ill-suited for robust evaluation under real-world conditions.

\subsection{Recommendation}
While bounded AUC-ROC can be shown to represent expected regret under the strong assumption of calibration, its reliance on model scores as implicit cost distributions and the strength of the calibration assumption make it unsuitable for most applications. Practitioners should instead use the Bounded Threshold Brier Score or Log Loss, which allow explicit specification of clinically or operationally relevant cost bounds without these problematic assumptions.

%% file: apdx/skillscore.tex
%
\section{Skill Scores: Interpretable Bounded Threshold Metrics}
\label{apdx:skillscore}

The newly proposed bounded-threshold metrics have a clear semantic benefit: they explicitly evaluate regret over the relevant range of decision contexts.  However, there are two aspects to metric interpretability: (1) clear semantics and (2) easy intuition.  Because the calculations are linear, scaling up or recentering our metrics will not change their ordering.  This offers us an opportunity.

One of the most helpful guides to interpreting a metric is the existence of anchors with clear meanings.
A skill score re-centers a metric so that 0.0 represents the performance of a naive baseline model, 1.0 represents perfect performance, and negative values indicate worse-than-guessing performance.
In general, guessing optimally based only on the population positive class prevalence $\pi$ without any features of the data is a good baseline; this is called the climatological baseline in meteorology.

For example, AUC-ROC of 1.0 means the model predicts everything correctly.
The baseline / climatological model would get 0.5, and 0.0 would mean every single label was flipped from the correct value.
So a Skill Score for AUC would be $2AUC - 1$.

Brier Score varies between 0.0 and 1.0, with 0.0 being the best possible performance and 1.0 requiring flipping every single label.
However, the naive / climatological forecast gets a different score depending on the prevalence - it can be shown to be the Bernoulli variance: $\pi(1-\pi)$.
This is the same normalizing factor that occurs in Theorem 4.2 in the connection between expected regret and AUC-ROC.

Thus if $BS$ is the Brier Score, then the Brier Skill Score is $1 - \frac{BS}{\pi(1-\pi)}$.
Brier Skill Score is 0.0 for the naive / climatological forecast, and 1.0 for perfect performance, and gives the fraction of the available expected value that is captured by the model.

Consider, for example, a model with $BS=0.1$, and $\pi=0.2$.  The baseline / climatological model would have a $BS=0.16$.  So the Brier Skill Score suggests that, averaged over the data and over cost ratios between 0 and 1, the model reduces regret by $\frac{0.06}{0.16} = 37.5\%$ when compared to the baseline without a model.

A similar expression exists for the Log Loss where the naive / climatological performance can be shown to be the Bernoulli entropy $H(\pi)=-\pi\log(\pi)-(1-\pi)\log(1-\pi)$.
If $LL$ is the Log Loss, then the Skill Score is $1 - \frac{LL}{H(\pi)}$.

\subsection{Bounded Brier Score Skill Score}

The naive / climatological performance for the Bounded Brier Score is (up to a normalizing factor) $E\left[(y-\text{clip}(\pi))^2 - (y-\text{clip}(y))^2\right]$.
Since none of these terms depend on the data, except through the prevalence, this is a valid normalizing factor.
We can use it to produce a Bounded Brier Skill Score that shows the fraction of baseline expected regret that is removed by using the classifier.

The same thing is possible for a Bounded Log Loss Skill Score.

In a scenario where the deployment cost of a classifier must be weighed against wholly separate options, it can be useful to rescale our cost matrix such that the loss L matches some quantitative estimate from experts.

In a scenario where we simply want to understand if the classifier is doing something useful, it can be preferable to rescale using the skill score, so that 1.0 is always the best possible classifier, and 0.0 the result of not using a classifier at all.

%% file: apdx/llm.tex
\section{LLM Literature Review}
\label{apdx:llm}

This appendix details our systematic approach to analyzing the use of evaluation metrics across machine learning research.
Our primary goal was to determine which metrics researchers prioritize when evaluating binary classifiers across different machine learning domains.
The findings provide important context for our main paper's recommendations on metric selection.

\begin{figure}[ht!]
\begin{center}
\includegraphics[width=.5\columnwidth]{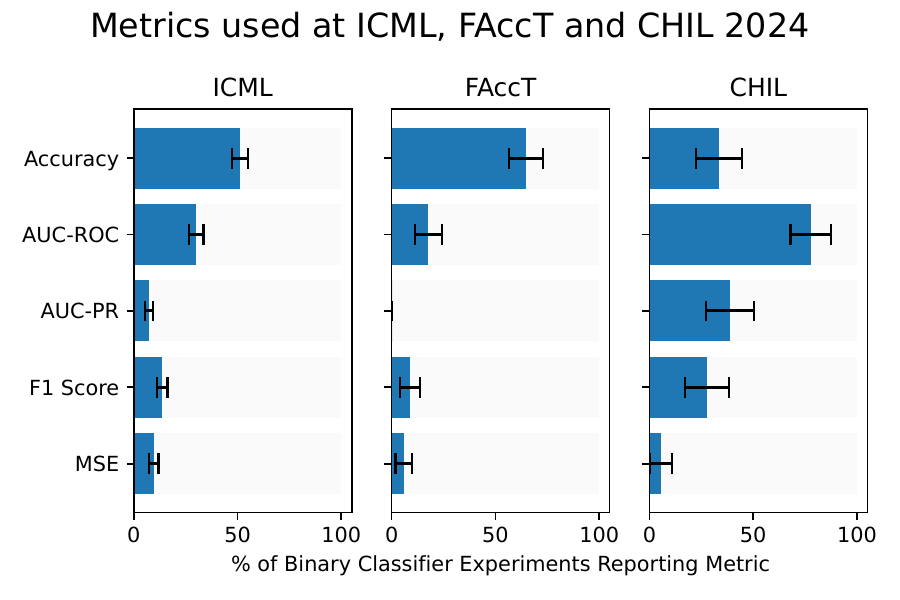}
\caption{
\small Claude 3.5 Haiku was used to analyze 2,610 papers from three major 2024 conferences.
Each plot summarizes the evaluation metrics used for binary classifiers.
Accuracy dominates outside healthcare, while AUC-ROC is more prevalent within healthcare domains.
Error bars come from binomial confidence intervals.
}
\label{fig:litreview}
\end{center}
\vskip -0.2in
\end{figure}

\subsection{Paper Acquisition}
Our data collection process focused on gathering papers from three major conferences in 2024: the International Conference on Machine Learning (ICML), the ACM Conference on Fairness, Accountability, and Transparency (FAccT), and the Conference on Health, Inference, and Learning (CHIL).
We developed automated scripts to acquire papers from their respective official sources:

\begin{itemize}
    \item ICML proceedings were accessed through OpenReview's conference platform: \url{https://openreview.net/group?id=ICML.cc/2024/Conference#tab-accept-oral}
    \item FAccT papers were obtained from the conference's official website: \url{https://facctconference.org/2024/acceptedpapers}
    \item CHIL proceedings were collected from the Proceedings of Machine Learning Research (PMLR): \url{https://proceedings.mlr.press/v248/}
\end{itemize}

For text extraction, we employed PyPDF2, a Python-based PDF processing library, to convert all acquired papers from PDF format to plain text.

\subsubsection{Classifier Identification}
We utilized Anthropic's Claude 3.5 Haiku model ("claude-3-5-haiku-20241022") to search the corpus for papers that mention binary classifiers.
The following prompt was sent to Anthropic's API along with the extracted text of each paper.

\begin{spverbatim}
    You are an AI assistant specializing in analyzing research papers in the field of machine learning and data science. Your task is to examine a given research paper and analyze its experimental methodology.

    Here is the research paper you need to analyze:

    <research_paper>
    {{RESEARCH_PAPER}}
    </research_paper>

    Please follow these steps to analyze the paper:

    1. Classifier Detection:
       - Determine if the paper involves a classifier.
       - If yes, identify whether it's binary, multiclass, or multilabel.
       - If no, explain why and continue to the next step.

    2. Experiment Detection:
       - Check if the paper includes experimental results.
       - If no, explain why and continue to the next step.

    3. Metric Analysis:
       - Identify which of the following metrics are reported:
         a) Classification metrics: Recall, Precision, F1, Accuracy
         b) Probabilistic metrics: Brier Score, Log Loss, Cross Entropy, Perplexity
         c) Error metrics: MSE, RMSE
         d) Cost/benefit metrics: Net Cost, Net Benefit
         e) Curve-based metrics: AUC-ROC, AUC-PR

       Important: Use only these exact metric names in your analysis. For example, use "AUC-ROC" instead of "AUROC", "AUC", or "AUCROC".

    4. Visualization Analysis:
       - Check for the inclusion of these visualizations:
         a) ROC curves
         b) Precision-Recall curves
         c) Brier curves
         d) Decision curves

    5. Summary:
       - Provide a JSON object summarizing your findings.

    For each step, wrap your thought process in <analysis_breakdown> tags before providing the final answer. In your analysis breakdown:
    - For classifier detection: List relevant quotes indicating the presence or absence of a classifier. Classify each quote as supporting binary, multiclass, or multilabel classification.
    - For experiment detection: List relevant quotes indicating the presence or absence of experiments. Summarize the type of experiments.
    - For metric and visualization analysis: Create a checklist of all possible metrics and visualizations mentioned in the instructions. Check them off one by one, citing relevant quotes for each.

    Use the following tags for your responses:
    <classifier> : Answer classifier-related questions
    <experiments> : Answer experiment-related questions
    <metrics> : List reported metrics
    <curves> : Indicate included visualizations
    <summary> : Provide the JSON summary

    Important guidelines:
    - Continue the analysis even if the paper doesn't involve classifiers or experiments.
    - Keep explanations concise (maximum 30 words for context in the JSON summary).
    - Include relevant quotes from the paper to support your findings.

    The JSON summary should follow this structure:
    {
        "has_classifier": boolean,
        "classifier_type": "none" | "binary" | "multiclass" | "multilabel",
        "has_experiments": boolean,
        "metrics": {
            "metric_name": {
                "present": boolean,
                "context": "Brief explanation (max 30 words)"
            }
        },
        "visualizations": {
            "visualization_name": boolean
        }
    }

    metric_name must be one of the following:
    "accuracy" | "auc" | "recall" | "precision" | "f1_score" | "mse" | "auprc" | "cross_entropy" | "rmse" | "mae" | "kl_divergence" | "average_precision"

    Please begin your analysis now.
\end{spverbatim}

We then extracted the JSON summary from the model's response, and found the headline results: that accuracy was dominant at ICML and FAccT, and that AUC-ROC was far more popular at CHIL.

After analyzing 2610 papers across the three conferences, we found significant differences in metric usage patterns:

\begin{itemize}
    \item At ICML and FAccT, \textbf{accuracy} was the dominant evaluation metric for binary classifiers, used in approximately 55.8\% and 61.3\% of relevant papers respectively.
    \item At CHIL, \textbf{AUC-ROC} was significantly more popular, appearing in 78.8\% of papers with binary classifiers, compared to accuracy at 33.6\%.
    \item AUC-PR was reported in 8.7\% of ICML papers, only 2.9\% of FAccT papers, but 27.7\% of CHIL papers, showing domain-specific preferences.
    \item  All other metrics were reported in less than 25\% of papers across the board.
\end{itemize}

These findings suggest substantial domain-specific differences in evaluation practices, particularly between general machine learning and healthcare applications.

\subsection{Second Check: More Powerful LLM}
We utilized Anthropic's Claude 3.5 Sonnet model ("claude-3-5-sonnet-20241022") to search those papers identified by Haiku as containing binary classifiers.  The following prompt was sent to Anthropic's API along with the extracted text of each paper.

\begin{spverbatim}
    You are an AI assistant specializing in analyzing research papers in the field of machine learning and data science. Your task is to examine the research paper given in the previous message, and analyze its experimental methodology.

    Please follow these steps to analyze the paper:

    1. Classifier Detection:
       - Determine if the paper involves a classifier.
       - If yes, identify whether it's binary, multiclass, or multilabel.
       - If no, explain why and continue to the next step.

    2. Experiment Detection:
       - Check if the paper includes experimental results.
       - If no, explain why and continue to the next step.

    3. Metric Analysis:
       - Identify which of the following metrics are reported:
         a) Classification metrics: Recall, Precision, F1, Accuracy
         b) Probabilistic metrics: Brier Score, Log Loss, Cross Entropy, Perplexity
         c) Error metrics: MSE, RMSE
         d) Cost/benefit metrics: Net Cost, Net Benefit
         e) Curve-based metrics: AUC-ROC, AUC-PR

       Important: Use only these exact metric names in your analysis. For example, use "AUC-ROC" instead of "AUROC", "AUC", or "AUCROC".

    4. Top-K Check:
       - When examining a binary edge classification task evaluated with AUC-ROC, how do we determine if there are constraints on the number of positive predictions allowed?
       -  For example, is the classifier free to predict any number of positives, or must it select exactly K edges?
       - What textual indicators or experimental details should we look for in the methodology to understand these constraints?

    5. Summary:
       - Provide a JSON object summarizing your findings.

    For each step, wrap your thought process in <analysis_breakdown> tags before providing the final answer. In your analysis breakdown:
    - For classifier detection: List relevant quotes indicating the presence or absence of a classifier. Classify each quote as supporting binary, multiclass, or multilabel classification.
    - For experiment detection: List relevant quotes indicating the presence or absence of experiments. Summarize the type of experiments.
    - For metric analysis: Create a checklist of all possible metrics mentioned in the instructions. Check them off one by one, citing relevant quotes for each.

    Use the following tags for your responses:
    <classifier> : Answer classifier-related questions
    <experiments> : Answer experiment-related questions
    <metrics> : List reported metrics
    <summary> : Provide the JSON summary

    Important guidelines:
    - Continue the analysis even if the paper doesn't involve classifiers or experiments.
    - Keep explanations concise (maximum 30 words for context in the JSON summary).
    - Include relevant quotes from the paper to support your findings.

    The JSON summary should follow this structure:
    {
        "has_classifier": boolean,
        "has_experiments": boolean,
        "decision_type": "independent" | "top-k" | "other" | "unknown",
        "metrics": {
            "metric_name": {
                "present": boolean,
                "context": "Brief explanation (max 30 words)"
            }
        }
    }

    Please begin your analysis now.
\end{spverbatim}

The Sonnet analysis confirmed our initial findings regarding metric preferences and suggested decision scenarios were overwhelmingly independent decisionmaking.
However, we left out the decision type from our summary because it required more complex reasoning.

\subsection{Human Spot Checks}

We spot checked a handful of papers to make sure that the model was accurately reporting the metrics being used.
We found that reporting of accuracy, AUC-ROC and AUC-PR was good.
Reporting whether Precision and Recall were being used directly as metrics, versus being mentioned in the context of AUC-PR, was sometimes a judgment call.
In one case, arguably Mean Squared Error was being used as a loss rather than an evaluation metric, since overall evaluation was not based on model quality.

\subsection{Conclusion}

This analysis reveals significant differences in how various research communities evaluate binary classifiers.
CHIL's preference for AUC-ROC aligns with healthcare's historical connection to ranking metrics, while ICML and FAccT researchers favor accuracy, reflecting their diminished focus on actual costs.
These findings inform our main paper's recommendations, showing that consequentialist evaluation remains a niche practice.